%% file: ijcai18.tex
\newcommand{\compilefullversion}{true} %SHOW full version

\documentclass{article}
\usepackage{ijcai18}

\usepackage{times}
\usepackage[dvipsnames]{xcolor}
\usepackage{soul}
\usepackage[utf8]{inputenc}
\usepackage[small]{caption}

\usepackage{ifthen}

\usepackage{graphicx}
\graphicspath{{figure/}}
\usepackage{xcolor}
    \definecolor{red}{HTML}{E51400} %red    
    \definecolor{blue}{HTML}{0050EF} %cobalt
    \definecolor{green}{HTML}{008A00} %emerald
    \definecolor{purple}{HTML}{AA00FF} %violet
    \definecolor{orange}{HTML}{FF7F00}
    \definecolor{gray}{HTML}{848482} 
\usepackage{setspace}
\usepackage{hyperref}
\hypersetup{colorlinks=true, linkcolor=blue, citecolor=red, anchorcolor=blue, urlcolor=blue}  % link color
\usepackage{epstopdf}   % Compile with pdflatex even if only .eps fig is given
\usepackage[ruled,vlined,linesnumbered]{algorithm2e}
\usepackage[symbol]{footmisc}

\usepackage{amsthm} % including proof
\newtheorem{fact}{Fact}
\newtheorem{theorem}{Theorem}
\newtheorem{lemma}{Lemma}

\newtheorem{definition}{Definition}
\usepackage{thm-restate}

\usepackage{amssymb}
\usepackage{amsmath}
\DeclareMathOperator{\E}{\mathbb{E}}
\DeclareMathOperator{\Var}{\mathrm{Var}}

\DeclareMathOperator{\rad}{\mathrm{rad}}
\DeclareMathOperator*{\argmax}{arg\,max}
\DeclareMathOperator*{\argmin}{arg\,min}
\newcommand{\norm}[1]{\lVert #1 \rVert}

\newcommand{\abs}[1]{|#1|}
\newcommand{\Abs}[1]{\left| #1 \right|}
\renewcommand{\vec}{\boldsymbol}
\newcommand{\alg}[1]{\textnormal{\textsc{#1}}}
\makeatletter
\newcommand*{\rom}[1]{\expandafter\@slowromancap\romannumeral #1@}
\makeatother

\ifthenelse{\equal{\compilefullversion}{false}}{%
    \newcommand{\OnlyInFull}[1]{}
    \newcommand{\OnlyInShort}[1]{#1}
}{%
    \newcommand{\OnlyInFull}[1]{#1}%
    \newcommand{\OnlyInShort}[1]{}%
}%

\newcommand{\compilehidecomments}{false}%HIDE comments

\ifthenelse{ \equal{\compilehidecomments}{true} }{%
	\newcommand{\wei}[1]{}
	\newcommand{\liang}[1]{}
	\newcommand{\huang}[1]{}
}{
	\newcommand{\wei}[1]{{\color{blue!50!black}  [\text{Wei:} #1]}}
	\newcommand{\liang}[1]{{\color{brown!60!black} [\text{Liang:} #1]}}
	\newcommand{\huang}[1]{{\color{magenta} [\text{Weiran:} #1]}}
}

\renewcommand\d{\mathrm{d}}
\newcommand\realtheta{\theta^*}
\newcommand\Realtheta{\vec \theta^*}

\newcommand \cB{\mathcal{B}}
\newcommand \cY{\mathcal{Y}}
\newcommand \width{{\rm width}}
\newcommand \Exchange{{\rm Exchange}}
\newcommand \g{\phi}
\newcommand \hardL{{\bf H_\Lambda}}
\newcommand \hardD{{\bf H_\Delta}}
\newcommand \hardU{{\bf H_\Lambda^U}}
\newcommand \monotonicity{{bi-monotonicity}} % double monotonicity
\OnlyInShort{\newcommand{\citet}[1]{\citeauthor{#1}~\shortcite{#1}}}
\OnlyInFull{\usepackage[sort&compress,numbers]{natbib}}

\title{Combinatorial Pure Exploration with Continuous and Separable \\ Reward Functions and Its Applications\OnlyInFull{ (Extended Version)}$^*$}

\author{
Weiran Huang$^{1,\dagger}$, 
Jungseul Ok$^2$, 
Liang Li$^3$,
Wei Chen$^{4,}$\thanks{Corresponding authors.} 
\\ 
$^1$ Tsinghua University,
$^2$ KTH,
$^3$ Ant Financial Group,
$^4$ Microsoft Research\\
huang.inbox@outlook.com,
jungseul@kth.se,
liangli.ll@antfin.com,
weic@microsoft.com
}
\begin{document}

\maketitle

\renewcommand{\thefootnote}{\fnsymbol{footnote}}
\footnotetext[1]{Due to the space constraints, supplementary materials and complete proofs are moved into the \OnlyInShort{extended version \cite{full}}\OnlyInFull{appendix}.  This work was supported in part by the National Basic Research Program of China Grant 2011CBA00300, 2011CBA00301, the National Natural Science Foundation of China Grant 61033001, 61361136003, 61433014. }
\renewcommand{\thefootnote}{\arabic{footnote}}
\OnlyInFull{\thispagestyle{plain}\pagestyle{plain}}

\input{content/abstract}
\input{content/introduction}
\input{content/problem_definition}
\input{content/general}
\input{content/applications}

\input{content/conclusion}

\begin{spacing}{0.95}
{
    \OnlyInFull{\bibliographystyle{unsrtnat}}
    \OnlyInShort{\bibliographystyle{named}}
    \bibliography{ref}
}
\end{spacing}

\OnlyInFull{
    \input{content/appendix}

}

\end{document}

%% file: content/abstract.tex
\begin{abstract}

We study the Combinatorial Pure Exploration problem with Continuous and Separable reward functions (CPE-CS)
    in the stochastic multi-armed bandit setting.
In a CPE-CS instance, we are given several stochastic arms with unknown distributions, 
    as well as a collection of possible decisions. 
Each decision has a reward according to the distributions of arms. 
The goal is to identify the decision with the maximum reward, using as few arm samples as possible.
The problem generalizes the 
    combinatorial pure exploration problem with linear rewards, 
    which has attracted significant attention in recent years.
In this paper, we propose an adaptive learning algorithm for the CPE-CS problem, and 
	analyze its sample complexity.
In particular, we introduce a new hardness measure called the consistent optimality hardness, and give both the upper and lower bounds of sample complexity.
Moreover,
    we give examples to demonstrate that our solution has the capacity to deal with non-linear reward functions.

\end{abstract}

%% file: content/introduction.tex
\section{Introduction}

The stochastic multi-armed bandit model is a predominant model for characterizing the trade-off between exploration and exploitation in a variety of application fields with stochastic environments.
In this model, we are given a set of stochastic arms associated with unknown distributions. 
Upon each play of an arm, the player can get a reward sampled from the corresponding distribution. 
The most well studied objective is to maximize the cumulative reward, or minimize the cumulative regret, e.g., \cite{lai1985asymptotically,auer2002nonstochastic,auer2002finite,bubeck2012regret}.
Another popular objective is to identify the optimal arm with high probability  by
	adaptively sampling arms based on the feedback collected.
This is called the pure exploration version of the multi-armed bandit problem~\cite{bubeck2010pure,audibert2010best,gabillon2012best}.

Instead of identifying the single optimal arm, there are a class of extended problems identifying the optimal combinatorial decision, 
    e.g., top-$k$ arm identification~\cite{kalyanakrishnan2010efficient,kalyanakrishnan2012pac,bubeck2013multiple,kaufmann2013information,zhou2014optimal}, 
    multi-bandit best arm identification~\cite{NIPS2011_4478}, 
    and their extension, Combinatorial Pure Exploration with Linear reward functions (CPE-L)~\cite{chen2014combinatorial,ChenGL16}, etc.
In CPE-L~\cite{chen2014combinatorial}, the rewards are linear
    functions on the means of underlying arms, and the decision class is subsets of arms satisfying certain 
    combinatorial constraints. 

In this paper, we further generalize CPE-L problems 
	to a large class of Combinatorial Pure Exploration with Continuous and Separable reward functions (CPE-CS)
	(see Section~\ref{sec:def} for the technical definition).
We propose the Consistently Optimal Confidence Interval (COCI) algorithm to solve the CPE-CS problem.
To analyze its sample complexity, we define a new arm-level measure called {\em consistent optimality radius}
	$\Lambda_i$ of arm $i$ and a new hardness measure called {\em consistent optimality hardness} 
	$\hardL = \sum_{i=1}^m 1/\Lambda_i^2$, where $m$ is the number of arms.
We prove that with probability at least $1-\delta$, COCI finds the optimal solution 
    in $O(\hardL\log (\hardL {\delta}^{-1}))$ rounds.
We also show that CPE-CS problems have a lower bound $\Omega(\hardL + \hardL m^{-1}\log\delta^{-1})$ in expectation,
	indicating that the hardness $\hardL$ is necessary.

We demonstrate the usefulness of CPE-CS by two applications. 
The first one is water resource planning \cite{bradley1977applied}.
The goal is to remove waste at water sources of an area.
One can first do some purification tests at different sources to estimate the water quality responses, 
    and then determines the final allocation of purification powers among different sources.
One need to balance the trade-off between the purification power and the cost, and usually the objective function is non-linear.
This application can be generalized to other urban planning scenarios such as air pollution control, crime control, etc.
The second application is partitioned opinion sampling \cite{bethel1986optimum,ballin2013joint,huang2017partitioned}.
The opinion polling is done by partitioning people into groups and sampling each group separately with different sample budget to improve
	the sample quality.
One can first do some tests in each group to estimate its opinion variance,
    and then determines the sample size for each group under the total sample budget for the formal sampling process.
In this case, the objective function is also non-linear.
Furthermore, we show that the COCI algorithm also solves the CPE-L problem with the same sample
	complexity as the CLUCB algorithm proposed by \citet{chen2014combinatorial}.

In summary, our contributions include: 
	(a) studying the combinatorial pure exploration problem with continuous and separable functions and proposing the COCI algorithm as its solution,
    (b) analyzing the sample complexity of COCI and providing both its lower and upper bounds with a novel hardness measure,
	and (c) applying the CPE-CS framework to water resources planning and partitioned opinion sampling with non-linear reward 
	functions to demonstrate the usefulness of the CPE-CS framework and the COCI algorithm.

{\bf Related Work.}
Pure exploration bandit studies adaptive learning methods to identify the optimal solution.
Best arm identification~\cite{bubeck2010pure,audibert2010best,gabillon2012best}, 
    top-$k$ arm identification~\cite{kalyanakrishnan2010efficient,kalyanakrishnan2012pac,bubeck2013multiple,kaufmann2013information,zhou2014optimal}, 
    the multi-bandit best arm identification~\cite{NIPS2011_4478} have been studied in the literature.
\citet{chen2014combinatorial,ChenGL16} generalize these studies to Combinatorial Pure Exploration with Linear reward functions (CPE-L).
\citet{soare2014best} also study the linear reward functions, but the player is required to select a decision to play instead of a single arm to sample in each round.
%CPE-L contains almost all previously studied CPE problems.
A very recent paper~\cite{ChenGLQW17} studies the CPE problems beyond linear reward functions, but
	their model assumes arms with Gaussian distributions and only works with the mean estimator, 
    while our CPE-CS only requires bounded distributions and also works for variance estimators.
Moreover, for efficient implementations, they need a pseudo-polynomial algorithm for the exact query
	besides the maximization oracle, but our solution only needs a maximization oracle.

A related online learning problem is multi-armed bandit (MAB), e.g., \cite{lai1985asymptotically,auer2002nonstochastic,auer2002finite,bubeck2012regret}.
The goal of MAB is to maximize cumulative rewards over multiple rounds, and the key is to balance
	exploration and exploitation during the learning process.
In contrast, in pure exploration, the key is the adaptive exploration in the learning process
	to quickly find the optimal solution, and thus it is fundamentally different from MAB~\cite{bubeck2010pure}.
Combinatorial MAB is a popular topic in recent years
\cite{cesa2012combinatorial,Yi2012,CWYW16,ChenGeneral16,GMM14,KWAEE14,KWAS15,Combes2015},
	but their goals and techniques are very different from ours.
    
%Partitioned opinion sampling (in some cases also referred as stratified sampling) has been extensively studied in the literature
%\cite{bethel1986optimum,bethel1989sample,chromy1987design,cochran2007sampling,kozak2007modern,keskinturk2007genetic,ballin2013joint,CarpentierM11,carpentier2012minimax,etore2010adaptive,huang2017partitioned}.
%Most studies focus on the offline problem of how to partition the population into groups or how to allocate samples for groups
%or both \cite{bethel1989sample,ballin2013joint,huang2017partitioned,keskinturk2007genetic,cochran2007sampling}.
%In term of sample allocation, we do not find an algorithm for the optimal integral solution.
%Either real-valued solutions or their simple roundings to the closed integers are used 
%\cite{CarpentierM11,etore2010adaptive}.
%But as we will show, these simple roundings may not give the optimal solution.
%A few studies investigate online and adaptive partitioned
%sampling~\cite{CarpentierM11,carpentier2012minimax,etore2010adaptive}, but their objectives are
%to minimize the sample variance from {\em all} samples collected during the learning process,
%while we use adaptive samples to find the optimal solution and those
%samples {\em do not affect} the final sample variance.
%These two are quite different and the difference is similar to the difference between pure exploration bandit and multi-armed bandit discussed below.

%% file: content/problem_definition.tex
\section{Problem Definition} \label{sec:def}

An instance of combinatorial pure exploration bandit problems consists of 
(a) a set of $m$ arms $[m]=\{1,\ldots, m\}$, each arm $i$ being 
associated with an unknown distribution $D_i$ with range $[0,1]$ and a key unknown parameter
$\realtheta_i\in[0,1]$ of $D_i$, 
(b) a finite set of decisions $\cY\subseteq \mathbb{R}^m$, with each decision $\vec{y}=(y_1, \ldots, y_m)$ as a
vector, and
(c) a real-valued (expected) reward function $r(\vec{\theta}; \vec{y})$ 
with vector $\vec{\theta}$ taken from the parameter space $[0,1]^m$ 
and $\vec{y}\in \cY$.
In each round $t= 1, 2, \dots$, a player selects one arm $i\in [m]$ to play, and observes a sample
independently drawn from $D_i$ as the feedback.
The player needs to decide based on the observed feedback so far if she wants to continue to play arms. 
If so, she needs to decide which arm to play next; if not, she needs to output a decision $\vec{y}^o \in \cY$ such that
with high probability $\vec{y}^o$ is the optimal decision maximizing the reward $r(\Realtheta; \vec{y}^o)$, where
$\Realtheta = (\realtheta_1, \ldots, \realtheta_m)$ is the vector of
the true underlying parameters of the unknown
distributions $\vec{D} = (D_1, \ldots, D_m)$.

\begin{definition}%[Combinatorial Pure Exploration Problem]
    Given a combinatorial pure exploration instance $([m], \cY, r(\cdot ; \cdot), \vec{D}, \Realtheta)$ and 
    a confidence error bound $\delta$, 
    the {\em combinatorial pure exploration (CPE) problem} requires the design of an algorithm
    with the following components: 
    (a) a stopping condition, which decides whether the algorithm should stop in the current round,
    (b) an arm selection component, which selects the arm to play in the current round 
    when the stopping condition is false, and
    (c) an output component, which outputs the decision $\vec{y}^o$ when
    the stopping condition is true.
    The algorithm could only use $([m], \cY, r(\cdot; \cdot))$ and the feedback from previous rounds
    as inputs, and should guarantee that with probability at least $1-\delta$ the output $\vec y^o$ is an 
    optimal decision, i.e., $\vec{y}^o \in \argmax_{\vec{y}\in \cY} r(\Realtheta; \vec{y})$.
\end{definition}

A standard assumption for CPE problems is that the optimal decision under the true parameter vector $\Realtheta$
is unique, i.e., $\vec{y}^* = \argmax_{\vec{y}\in \cY} r(\Realtheta; \vec{y})$.
The performance of a CPE algorithm is measured by its {\em sample complexity}, which is the number of 
rounds taken when the algorithm guarantees its output to be the optimal one with
probability at least $1-\delta$.

We say that a reward function $r(\vec{\theta}; \vec{y})$ is {\em continuous} if 
$r(\vec{\theta}; \vec{y})$ is continuous in $\vec{\theta}$ for every $\vec{y} \in \cY$,
%We say that the reward function $r(\vec{\theta}; \vec{y})$ is 
and {\em (additively) separable} if
there exist functions $r_1, \ldots, r_m$ such that 
$r(\vec{\theta}; \vec{y})=\sum_{i=1}^m r_i(\theta_i, y_i)$.
We use CPE-CS to denote the class of CPE problems with Continuous and Separable reward functions and 
each parameter $\realtheta_i$  of arm $i$ can either be mean $\E_{X\sim D_i}[X]$ or 
variance $\Var_{X\sim D_i}[X]$.\footnote{Other parameter $\theta_i^*$ of $D_i$ is also acceptable 
	if it has an unbiased estimator from the samples of $D_i$. Only a minor change is needed
	in the formula of confidence radius in COCI (Algorithm~\ref{alg:general}).}
We use $\alg{Est}_i(X_{i,1}, X_{i,2},\ldots, X_{i,s})$ to denote the unbiased 
estimator for parameter 
$\realtheta_i$ from $s$	i.i.d.\ observations $X_{i,1}, X_{i,2},\ldots, X_{i,s}$ of the $i$-th arm.
In particular, for the mean estimator, 
$\alg{Est}_i(X_{i,1}, X_{i,2}, \dots, X_{i,s}) = \frac{1}{s} \sum_{j=1}^s X_{i,j}$,
and for the variance estimator, 
$\alg{Est}_i(X_{i,1}, X_{i,2}, \dots, X_{i,s})=\frac{1}{s-1} \left(\sum_{j=1}^{s}X_{i,j}^2-\frac{1}{s}(\sum_{j=1}^{s}X_{i,j})^2\right)$.
Notice that the variance estimator needs at least two samples.
We also define $\g\colon [0,1]^m \rightarrow \cY $ to be a deterministic tie-breaking maximization oracle such that for any $\vec \theta\in {[0,1]^m}$,
$\g(\vec\theta)=(\g_1(\vec\theta), \dots, \g_m(\vec\theta))\in \argmax_{\vec y\in \cY} r(\vec \theta; \vec y)$
and it always outputs the same optimal solution,
called the {\em leading optimal solution}, under the same parameter $\vec\theta$.

CPE-CS encompasses the important CPE problems with Linear reward
functions (CPE-L).
In CPE-L, parameter $\realtheta_i$ is the mean of arm $i$ for each $i\in[m]$.
Each decision is a subset of $[m]$, which can be represented as an $m$-dimensional binary vector.
Thus, the decision space $\cY$ is a subset of $\{0,1\}^m$, and each vector $\vec{y}=(y_1, \ldots, y_m)\in \cY$
represents a subset of arms $S_{\vec y} = \{i\in [m] \colon y_i = 1 \}$.
Moreover, the reward function $r(\vec{\theta}; \vec{y}) = \sum_{i=1}^m \theta_i \cdot y_i$ is continuous and separable.

%% file: content/general.tex
\section{Solving CPE-CS} \label{sec:cpecs}

In this section, we propose the Consistently Optimal Confidence Interval (\alg{COCI}) Algorithm for CPE-CS, and analyze its sample complexity.
En route to our sample complexity bound, we introduce a new concept of arm-level {\em consistently optimal radius} $\Lambda_i$ of
	each arm $i$, which leads to a new
	hardness measure $\hardL$.
%The introduction of the consistently optimal radius, together with the
%	handling of general continuous and separable reward functions, results in a more compact and streamlined
%	sample complexity analysis than
%	the one for CPE-L in~\cite{chen2014combinatorial}, and also provides appealing conceptual understanding
%	of the determining factors for the sample complexity. 
We first introduce the components and notations which will be used in the algorithm.

\newcommand\mycommfont[1]{\footnotesize\sl{#1}}
\SetCommentSty{mycommfont}
\begin{algorithm}[t] 
    \caption{\alg{COCI}: Consistently Optimal Confidence Interval Algorithm for CPE-CS}\label{alg:general}
    \KwIn{Confidence error bound $\delta\in(0,1)$, maximization oracle $\g$.}
    \KwOut{$\vec y^o = (y_1, y_2, \dots, y_m) \in \mathcal{Y}$. }
    
   % \tcp{Initialize confidence interval space $\hat\Theta_{\tau m}$}
    $t\gets\tau m$\tcp*{$\tau = 1$ for the mean estimator and $\tau=2$ for the variance estimator} \label{line:initestb}
    \For{$i = 1, 2, \dots, m$}{
                
        observe the $i$-th arm  $\tau$ times $X_{i,1}, \dots, X_{i,\tau}$\label{algline:initialize}\;%\tcp*{$\tau$ is the minimum number of samples to obtain a validate estimate $\hat{\theta}_{i,t}$}
        $T_{i,t} \gets \tau$\;
        
        estimate $\hat{\theta}_{i,t} \gets \alg{Est}_i(X_{i,1}, \dots, X_{i,T_{i,t}})$\;\label{algline:estimation}
        $\rad_{i,t} \gets \sqrt{\frac{1}{2T_{i,t}}\ln\frac{4t^3}{\tau \delta}}$\tcp*{confidence radius}
        $\hat\Theta_t \hspace{-0.1em}\gets \{\vec\theta\in[0,1]^m \colon |\theta_i-\hat\theta_{i,t}|\le \rad_{i,t}, \forall i\in[m]\}$\;\label{line:initeste}
    }
    
    \For{$t = \tau m+1, \tau m+2, \tau m+3, \dots$}{
        $C_{t} \gets \emptyset$\; \label{line:candidateb}
        \For{$i = 1, 2, \dots, m$}{
            \If{ $\max_{\vec\theta\in\hat\Theta_{t-1}} \phi_i(\vec\theta) \ne \min_{\vec\theta\in\hat\Theta_{t-1}} \phi_i(\vec\theta)$\label{algline:oracle}}{
               $C_{t} \gets C_{t}\cup\{i\}$\;\label{line:candidatee}
            }
        }
        \If{$C_{t}=\emptyset$ \label{line:stopcond}}{
        	{\bf return} $\vec{y}^o = \g(\vec{\theta})$ for an arbitrary $\vec{\theta}\in \hat\Theta_{t-1}$\;
        	\label{line:COCIoutput}
        }
        $j \gets \argmax_{i\in C_{t}}\rad_{i,t-1}$\; \label{line:largestradius}
	        %\tcp*{pick any one if there are more than one choices}
        $T_{j,t} \gets T_{j,t-1}+1$;
        $T_{i,t} \gets T_{i,t-1}$ for all $i\ne j$\; \label{line:updatejb}
        play the $j$-th arm and observe the outcome $X_{j,T_{j,t}}$\;
        update $\hat{\theta}_{j,t} \gets \alg{Est}_j(X_{j,1},\ldots, X_{j,T_{j,t}})$\;
	    update $\hat{\theta}_{i,t} \gets \hat{\theta}_{i,t-1}$ for all $i\ne j$\;
        update $\rad_{i,t} \gets \sqrt{\frac{1}{2T_{i,t}}\ln\frac{4t^3}{\tau \delta}}$ for all $i\in[m]$\;
        $\hat\Theta_t \hspace{-0.1em} \gets \{\vec\theta\in[0,1]^m \colon |\theta_i-\hat\theta_{i,t}|\le \rad_{i,t}, \forall i\in[m]\}$\; \label{line:updateje}
    }
\end{algorithm}

The algorithm we propose for CPE-CS (Algorithm~\ref{alg:general}) is based on the confidence intervals
	of the parameter estimates.
The algorithm maintains the confidence interval space $\hat\Theta_{t}$ for every round $t$ to
	guarantee that the true parameter $\Realtheta$ is always in $\hat\Theta_{t}$ for all $t>0$ with probability
	at least $1-\delta$.
%The algorithm first initialize the parameter estimates $\hat{\theta}_{i,t}$'s and their confidence
%	radii $\rad_{i,t}$'s (lines~\ref{line:initestb}--\ref{line:initeste}).
After the initialization (lines~\ref{line:initestb}--\ref{line:initeste}), 
	in each round $t$, the algorithm first computes the candidate set $C_t\subseteq [m]$
	(lines~\ref{line:candidateb}--\ref{line:candidatee}).
According to the key condition in line~\ref{algline:oracle}, $C_t$ contains the $i$-th arm 
	if $\max_{\vec\theta\in\hat\Theta_{t-1}} \phi_i(\vec\theta) \ne \min_{\vec\theta\in\hat\Theta_{t-1}} \phi_i(\vec\theta)$ (this is a logical condition, and its actual implementation will be discussed
	in Section~\ref{sec:bimonotonicity}).
%Since the true parameter $\Realtheta$ is in $\hat\Theta_{t-1}$ with high confidence,
%	the above condition means that within $\hat\Theta_{t-1}$,  
%	the leading optimal solutions are still not consistent with the true optimal one $\g(\Realtheta)$ 
%	in the $i$-th dimension.
The stopping condition is $C_t =\emptyset$ (line~\ref{line:stopcond}),
	which means that within the confidence interval space,
	all leading optimal solutions are the same. % with the true optimal one
	%(since the true parameter $\Realtheta$ is in $\hat\Theta_{t-1}$ with high confidence).
In this case, the algorithm returns the leading optimal solution under any $\vec\theta\in \hat\Theta_{t-1}$ as the final output (line~\ref{line:COCIoutput}).
Notice that if the true parameter $\vec\theta^*$ is in $\hat\Theta_{t-1}$, then the output is the true optimal solution $\vec y^o=\g(\Realtheta)=\vec y^*$.
If $C_t \ne \emptyset$, then the algorithm picks any arm $j$ with the largest
	confidence radius (line~\ref{line:largestradius}), plays this arm, observes its feedback, and
	updates its estimate $\hat{\theta}_{j,t}$ and confidence radius $\rad_{j,t}$ accordingly
	(lines~\ref{line:updatejb}--\ref{line:updateje}).
Intuitively, arm $j$ is the most uncertain arm causing inconsistency, thus the algorithm picks it to play first.
Since the key stopping condition is that the leading optimal solutions for all $\vec{\theta}\in \hat{\Theta}_{t-1}$ 
	are consistently optimal, 
	we call our algorithm Consistently Optimal Confidence Interval (\alg{COCI}) algorithm.

Before analyzing the sample complexity of the \alg{COCI} algorithm,
    we first introduce the (arm-level) {\em consistent optimality
	radius} for every arm $i$, which is formally defined below.
\begin{definition}\label{def:Lambdai}
    For all $i\in[m]$, the {\em consistent optimality radius} $\Lambda_i$ for arm $i$ is defined as: 
%    \[
%    \Lambda_i = \sup_{ \substack {\Lambda: \forall \theta, \forall j\in[m], \abs{\theta_j-\realtheta_j}<\Lambda \\ \Rightarrow g_i(\theta)=g_i(\realtheta)}} \Lambda.
%    \]
%
    \begin{equation*}% \label{eq:Lambdai}
    \Lambda_i 
    =\inf_{\vec\theta: \g_i(\vec\theta) \neq \g_i(\Realtheta)} \left\| \vec\theta-\Realtheta \right\|_\infty.
    %= \inf_{\vec\theta: \g_i(\vec\theta) \neq \g_i(\Realtheta)} \max_{j\in[m] } \abs{\theta_j-\realtheta_j}.
    \end{equation*}
\end{definition}

Intuitively, $\Lambda_i$ measures how far $\vec \theta$ 
	can be away from $\Realtheta$ (in infinity norm)
	while the leading optimal solution under $\vec{\theta}$ is still consistent with the true optimal one
	in the $i$-th dimension, as precisely stated below.
%The following proposition precisely states
%	this intuition, and it is directly implied by the definition.
\begin{restatable}{proposition}{LambdaProposition}\label{proposition:Lambda}
    $\forall i\in [m]$, if $\abs{\theta_j-\realtheta_j}<\Lambda_i$ holds for all $j\in[m]$, then $\g_i(\vec\theta)=\g_i(\Realtheta)$.
\end{restatable}

The following lemma shows that the consistent optimality radii are all
	positive, provided by that the reward function is continuous and the true optimal decision $\vec y^*$ is unique.
\begin{restatable}{lemma}{LambdaPositive}
	If the reward function $r(\vec{\theta}; \vec{y})$ is continuous on $\vec{\theta}$ for every
	$\vec{y} \in \cY$, and
	the optimal decision $\vec y^*$ under the true parameter vector $\Realtheta$ is unique, then $\Lambda_i$ is positive for every $i\in[m]$.
\end{restatable}
Given that the consistent optimality radii are all positive, we can introduce the key hardness measure
	used in the sample complexity analysis.
We define {\em consistent optimality hardness} as
	$\hardL = \sum_{i=1}^m \frac{1}{\Lambda_i^2}$.
%In the next section, we will give the precise relationship between the consistent optimality hardness
%	$\hardL$ and the hardness measure defined for CPE-L.
The following theorem shows our primary sample complexity result for the COCI algorithm.
\begin{restatable}{theorem}{samplecomplexity} \label{thm:samplecomplexity}
	With probability at least $1-\delta$, the COCI algorithm (Algorithm~\ref{alg:general})
	returns the unique true optimal solution $\vec{y}^o = \vec{y}^*$, and the number of rounds (or samples) $T$ satisfies
	\begin{align} 
	T &\leq 2m+12\hardL\ln 24\hardL +4\hardL\ln\frac{4}{\tau\delta }\nonumber \\
	&= O\left( \hardL \log \frac{\hardL} {\delta}  \right).\label{eq:samplecomplexity}
	\end{align}
\end{restatable}

Theorem~\ref{thm:samplecomplexity} shows that the sample complexity is positively related
	to the consistent optimality hardness, or inversely proportional to the square of consistent
	optimality radius $\Lambda_i^2$.
Intuitively, when $\Lambda_i$ is small, we need more samples to make the optimal solutions in the
	confidence interval consistent on the $i$-th dimension, and hence higher sample complexity.
%Theorem~\ref{thm:samplecomplexity} also shows that the sample complexity is proportional 
%	to $\log(1/\delta)$, which is natural since higher confidence (lower $\delta$) requires more samples.
	
We remark that if we do not compute the candidate set $C_t$ and directly pick the arm with the largest radius
	among {\em all} arms in line~\ref{line:largestradius}, every arm will be selected in a round-robin fashion
	and COCI becomes a uniform sampling algorithm.
In \OnlyInFull{Appendix~\ref{app:uni}}\OnlyInShort{the extended version}, we show that the sample complexity 
	upper bound of the uniform version 
	is obtained by replacing $\hardL$ in Eq.~\eqref{eq:samplecomplexity} by
	$\hardU = \frac{m}{\min_{i\in [m]} \Lambda_i^2}$, and the factor $\hardU$ is tight for the uniform sampling.
This indicates that the adaptive sampling method of COCI would perform much better than the uniform sampling
	when arms have heterogeneous consistent optimality radii such that $\hardL \ll \hardU$.

%Sample complexity depends on the two parameters.
%One is confidence error bound $\delta$.
%When the confidence error bound $\delta$ gets smaller, the sample complexity becomes larger, which means the player needs more rounds to obtain the optimal decision with higher probability.
%The other one is intrinsic hardness $\hardL$, which represents the intrinsic hardness of the problem.
%If some $\Lambda_i$ is really small, then $\hat \theta_i$ needs to be estimated enough accurate to make the COCI algorithm stop.
%Thus the sample complexity will be large.

Due to the space constraint, we only provide the key lemma below leading to the proof of the theorem.
We define a random event 
$\xi = \{\forall t\ge \tau m, \forall i\in [m], \abs{\hat\theta_{i,t}-\realtheta_i}\le\rad_{i,t}\}$,
which indicates that $\Realtheta$ is inside the confidence interval space of all the rounds.
%
%For each $t\ge \tau m$, we define the random event 
%	$\xi_t = \{\forall i\in [m], \abs{\hat\theta_{i,t}-\realtheta_i}<\rad_{i,t}\}$,
%	and define event $\xi:=\bigcap_{t=\tau m}^\infty\xi_t$.
%That is, event $\xi_t$ indicates that in round $t$ the true parameter $\Realtheta$ is inside the confidence interval $\hat{\Theta}_t$, and
Then we have the following lemma.
\begin{restatable}{lemma}{radius}\label{lem:radius}
    Suppose event $\xi$ occurs. For every $i\in[m]$ and every $t>\tau m$,
    if $\rad_{i,t-1} < \Lambda_i/2$, then the $i$-th arm will not be played in round $t$.
\end{restatable}
\begin{proof}
	Suppose, for a contradiction, that the $i$-th arm is played in round $t$, namely, $i\in C_t$, and
	$i=\argmax_{j\in C_{t}}\rad_{j,t-1}$.
	Thus for each $j\in C_t$, we have $\rad_{j,t-1} \leq \rad_{i,t-1} <\Lambda_i/2$.

	We claim that for all $\vec\theta \in \hat{\Theta}_{t-1}$, $\g_i(\vec\theta)=\g_i(\Realtheta)$.
	If so, $\max_{\vec\theta\in\hat\Theta_{t-1}} \g_i(\vec\theta) = \min_{\vec\theta\in\hat\Theta_{t-1}} \g_i(\vec\theta)$,
	then by line~\ref{algline:oracle} $i \not\in C_t$, a contradiction.
	
	We now prove the claim.
	For any vector $\vec{x} \in \mathbb{R}^m$ and any index subset $C\subseteq [m]$, we use $\vec{x}_C$ to denote the sub-vector of $\vec{x}$ projected onto $C$.
	For vector-valued functions such as $\g(\vec\theta)$, we use $\g_{C}(\vec\theta)$ for $\g(\vec\theta)_{C}$.
	For any $\vec\theta \in \hat{\Theta}_{t-1}$, we construct an intermediate 
	vector $\vec \theta' = (\vec{\theta}_{C_t}, \Realtheta_{-C_t})$, i.e., the $j$-th component $\theta'_j$ is $\theta_j$ when $j\in C_t$, or $\theta^*_j$ when $j\notin C_t$.
	Since event $\xi$ occurs, we have $|\hat \theta_{j,t-1} - \realtheta_j| \le \rad_{j,t-1}$ for $j\in[m]$.
    Thus for all $j\in C_t$,
	$|\theta'_j - \realtheta_j|\le|\theta_j - \hat \theta_{j,t-1} |+ |\hat \theta_{j,t-1} - \realtheta_j| \le 2 \rad_{j,t-1} < \Lambda_i $,
	and for all $j\notin C_t$, $|\theta'_j - \realtheta_j| = 0$.
	This means that $\norm{\vec \theta'-\Realtheta}_\infty < \Lambda_i$.
	According to Proposition~\ref{proposition:Lambda}, $\g_i(\vec \theta')=\g_i(\Realtheta)$.
	We next prove that $\g_i(\vec \theta)=\g_i(\vec \theta')$, 
	which directly leads to $\g_i(\vec\theta)=\g_i(\Realtheta)$.
	
	Since event $\xi$ occurs and $\vec\theta^*\in[0,1]^m$, $\Realtheta$ is in $\hat{\Theta}_{t-1}$.
    By the definition of $\vec \theta'$ and $\vec\theta\in\hat{\Theta}_{t-1}$, $\vec \theta'$ is also in $\hat{\Theta}_{t-1}$.
	According to Algorithm~\ref{alg:general}, for each $j\notin C_{t}$, we have $\max_{\vec\theta\in\hat\Theta_{t-1}} \g_j(\vec\theta) = \min_{\vec\theta\in\hat\Theta_{t-1}} \g_j(\vec\theta)$, 
	thus  $\g_{-C_t}(\vec\theta)=\g_{-C_t}(\vec\theta')=\g_{-C_t}(\Realtheta)$.

	%We assume that $\g$ is the first optimal action ordered by alphabet.
	Note that the reward function is separable, we have
	\begin{align*}
	r(\vec \theta; \vec y)&= \sum_{j\in C_t}r_j(\theta_j, y_j) + \sum_{j\notin C_t}r_j(\theta_j, y_j).
	\end{align*}
	Let $\mathcal{Y}_{C_t}(\vec \theta) = \{\vec y_{C_t} \colon \vec y\in \mathcal{Y} \land \vec y_{-C_t}=\g_{-C_t}(\vec \theta)\}$.
	It is straightforward to verify that $\g_{C_t}(\vec \theta)$ is the leading optimal solution
	for the following problem:
	\begin{align}
	\text{max}\quad &\sum_{j\in C_t}r_j(\theta_j, z_{j}),\nonumber \\
	\text{subject to}\quad &\vec z\in\mathcal{Y}_{C_t}(\vec \theta). \label{eq1}
	\end{align}	 	
	Similarly, we have 
	\begin{align*}
	r(\vec \theta'; \vec y) %= \sum_{j\in C_t}f_j(\theta_j', y_j) + \sum_{j\notin C_t}f_j(\theta_j', y_j) 
	= \sum_{j\in C_t}r_j(\theta_j, y_j) + \sum_{j\notin C_t}r_j(\realtheta_j, y_j),
	\end{align*}
	and $\g_{C_t}(\vec \theta')$ is the leading optimal solution for 
	\begin{align}
	\text{max}\quad &\sum_{j\in C_t}r_j(\theta_j, z_{j}), \nonumber\\
	\text{subject to}\quad &\vec z\in\mathcal{Y}_{C_t}(\vec \theta^*).\label{eq2}
	\end{align} 
	Since $\g_{-C_t}(\vec \theta)=\g_{-C_t}(\vec \theta^*)$, optimization problems~(\ref{eq1}) and (\ref{eq2})
%	 have the same feasible region, namely, $\mathcal{Y}_{C_t}(\vec \theta)=\mathcal{Y}_{C_t}(\vec \theta')$.
%	Thus, problem~(\ref{eq1}) and (\ref{eq2}) 
	are identical, thus they have the some leading optimal solution $\g_{C_t}(\vec \theta)=\g_{C_t}(\vec \theta')$.    
	Notice that $i\in C_t$, therefore, $\g_i(\vec \theta)=\g_i(\vec \theta')$ holds.
\end{proof}

The above lemma is the key connecting consistent optimality radius $\Lambda_i$ with confidence radius
$\rad_{i,t-1}$ and the stopping condition.
Its proof relies on both the definition of consistent optimality radius and the assumption of
separable reward functions.
With this lemma, the sample complexity can be obtained by considering the first round when every arm
satisfies the condition $\rad_{i,t-1} < \Lambda_i/2$.

Borrowing a lower bound analysis in~\cite{ChenGLQW17}, we can further
	show that the hardness measure $\hardL$ is necessary for
	CPE-CS, even CPE-L, as shown below.

\begin{restatable}{theorem}{lowerbound}\label{thm:lowerbound}
    Given $m$ arms and $\delta\in(0,0.1)$, 
    there exists an instance such that every algorithm $\mathcal{A}$ for CPE-L which outputs the optimal solution with probability at least $1-\delta$, takes at least $$\Omega(\hardL + \hardL m^{-1}\log\delta^{-1})$$ samples in expectation. 
\end{restatable}

\subsection{Implementing the Condition in Line~\ref{algline:oracle}}\label{sec:bimonotonicity}

The key condition in line~\ref{algline:oracle} of Algorithm~\ref{alg:general} is a logical one revealing  
the conceptual meaning of the stopping condition, but it does not lead to a direct implementation.
In many CPE-CS instances, 
the condition can be translated to a condition only on the boundary 
of $\hat\Theta_{t-1}$, and further due to the bi-monotonicity of $\g$ introduced below, it has an efficient
implementation.
Such instances include best-arm identification, top-$k$ arm identification, water resources planning (Section~\ref{sec:water}), partitioned opinion sampling (Section~\ref{sec:appcpess}), etc.

We say that the leading optimal solution $\g(\vec \theta)$ satisfies \emph{\monotonicity{}}, if for each $i\in [m]$,  $\g_i(\vec{\theta})$ is monotonically non-increasing (or non-decreasing) in $\theta_i$, and monotonically non-decreasing (or non-increasing) in $\theta_j$ for all $j\ne i$.
%i.e., 
%$\forall \vec \theta$, $\forall \theta_i'\geq \theta_i$, $\forall \theta_j'\leq \theta_j$ where $j\neq i$,
%$\g_i(\vec{\theta})\le \g_i(\vec{\theta}_{-i}, \theta'_i)$ and
%$\g_i(\vec{\theta})\le \g_i(\vec{\theta}_{-j}, \theta'_j)$.
For convenience, we use $\overline{\theta}_{i,t}= \max_{\vec\theta\in\hat\Theta_{t}}\theta_i$ and $\underline{\theta}_{i,t}=\min_{\vec\theta\in\hat\Theta_{t}}\theta_i$ to denote the upper and lower confidence bound of arm $i$ in round $t$.
We also use $\overline{\vec{\theta}}_{-i,t}$ and $\underline{\vec{\theta}}_{-i,t}$ to denote the upper and lower confidence bounds of all arms excluding arm $i$.

\begin{restatable}{theorem}{lucb}\label{thm:condition}
    If the leading optimal solution $\g(\vec \theta)$ satisfies \monotonicity{},
    the condition in line~\ref{algline:oracle} of Algorithm~\ref{alg:general} can be efficiently implemented by
    \begin{align*}
    \g_i(\underline{\vec{\theta}}_{-i,t-1}, \overline{\theta}_{i,t-1})\ne \g_i(\overline{\vec{\theta}}_{-i,t-1}, \underline{\theta}_{i,t-1}).
    \end{align*}
\end{restatable}

The above theorem indicates that, when bi-monotonicity holds for $\g(\vec \theta)$, we only need two calls
	to the offline oracle $\g(\vec \theta)$ to implement the condition in line~\ref{algline:oracle}, and thus
	the COCI algorithm has an efficient implementation in this case.

%% file: content/applications.tex
\section{Applications}\label{sec:application}
\input{content/watersystem}
\input{content/partitionedsampling}

\input{content/cpe-l}

%% file: content/watersystem.tex
\subsection{Water Resource Planning}\label{sec:water}

Water resource systems benefit people to meet drinking water and sanitation needs, and also support and maintain resilient biodiverse ecosystems. 
In regional water resource planning, one need to determine the Biological Oxygen Demand (BOD, a
measure of pollution) to be removed from the water system at each source.
Online learning techniques proposed in recent years make adaptive optimization for water resource planning possible.

Let $y_i$ be the pounds of BOD to be removed at source $i$.
One general model (adapted from \cite{bradley1977applied}) to minimize total costs to the region to meet specified pollution standards can be expressed as:
\begin{align}
\text{max}\quad &\sum_{i=1}^m \theta_i^* y_i -\sum_{i=1}^m f_i(y_i),\nonumber\\
\text{subject to}\quad &\sum_{i=1}^m y_i \geq b, 0\leq y_i\leq c_i, \forall i\in [m],\label{eq:water}
\end{align}
where 
$\theta_i^*$ is the quality response caused by removing one pound of BOD at source $i$, and
$f_i(y_i)$ is the cost of removing $y_i$ pounds of BOD at
source $i$.
Each $y_i$ is constrained by $c_i$, the maximum pounds of BOD that can be
removed at source $i$.
Moreover, the total pounds of BOD to be removed are required to be larger than a certain threshold $b$.

The above model formulates the trade-off between the benefit and the cost of removing the pollutants.
The cost function $f_i$ is usually known and non-linear, which may depend on the cost of oxidation,  labor cost, 
	facility cost, etc., 
while the quality response $\theta_i^*$ is unknown beforehand, and needs to be learned from
	tests at source $i$.
In each test, the tester measures the quality response at a source $i$ and gets an observation of $\theta_i^*$, 
	which can be regarded as a random variable $\theta_i$ derived from an unknown distribution with mean $\theta_i^*$.
The goal is to do as few tests as possible to estimate the quality responses, and 
	then give a final allocation $(y_1^o, \dots, y_m^o)$ of BOD among sources as the plan to be implemented
	(e.g., building BOD removal facilities at the sources).
	
The above problem falls into the CPE-CS framework.
The $i$-th source corresponds to the $i$-th arm. 
Each quality response at source $i$ is the unknown parameter $\theta_i^*$ associated with the arm $i$, and $\tau=1$.
Each allocation $(y_1, \dots, y_m)$ satisfying the constraints corresponds to a decision.
We discretize $\{y_i\}$'s so that the decision class $\cY$ is finite.
The reward function is $r(\vec\theta,\vec y)= \sum_{i=1}^m \theta_i y_i -\sum_{i=1}^m f_i(y_i)$, which is continuous and separable.
Suppose the offline problem of Eq.~(\ref{eq:water}) when $\vec\theta^*$ is known can be solved by a known oracle $\phi(\vec\theta^*)$.
Then, the COCI algorithm can be directly applied to the water resource planning problem.
The following lemma gives a sufficient condition for the bi-monotonicity of $\phi$.

\begin{restatable}{lemma}{waterbi}
When \{$\d f_i/\d y_i$\}'s are all monotonically increasing or decreasing, and the constraint $\sum_{i=1}^m y_i \geq b$ is tight at the leading optimal solution $\g(\vec{\theta})$ for all $\vec{\theta}$, 
	then $\g(\vec{\theta})$ satisfies \monotonicity{}.
\end{restatable}

%\wei{I revised the above theorem, explicitly mentioning 
%		``tight at the leading optimal solution $\g(\vec{\theta})$ for all $\vec{\theta}$''.
%	Please check if this statement is accurate.
%}

By Theorem~\ref{thm:condition}, when the offline oracle for the water resources planning problem satisfies bi-monotonicity, we can instantiate the condition in line \ref{algline:oracle} of Algorithm~\ref{alg:general} as $\g_i(\underline{\vec{\theta}}_{-i,t-1}, \overline{\theta}_{i,t-1}) \ne \g_i(\overline{\vec{\theta}}_{-i,t-1}, \underline{\theta}_{i,t-1})$.

Although this application is set up in the context of water resource planning, we can see that the formulation
	in Eq.~\eqref{eq:water} is general enough to model other applications, especially ones in
	the urban planning context.
For example, for planning air quality control for a city, we need to target a number of air pollution emission sources
	(e.g., factories),
	and do adaptive testing at the sources to determine the optimal pollution remove target at
	each sources which maximizes the total utility of the planning.
Other applications, such as crime control, may also be modeled similarly as instances of
	our CPE-CS framework and solved effectively by our COCI algorithm.

%% file: content/partitionedsampling.tex
\subsection{Partitioned Opinion Sampling} \label{sec:appcpess}

Public opinion dynamics has been well studied, and 
    there are a number of opinion dynamic models proposed in the literature,
    such as the voter model \cite{VoterModel}, and its variants \cite{DiscreteOpinion,Li0WZ15,huang2017partitioned}.
In these models, 
    people's opinions $f_1^{(t)}, f_2^{(t)}, \dots, f_n^{(t)} \in [0,1]$ change over time $t$,
    and will converge to a steady state after sufficient social interactions in which the joint distribution of people's opinions no longer changes.
Thus, they are regarded as Bernoulli random variables %$F_1$, $F_2$, \dots, $F_n$
    derived from the steady-state joint distribution,
    and sampling at time $t$ can be considered as observing part of a realization of $f_1^{(t)}, f_2^{(t)}, \dots, f_n^{(t)}$.
In partitioned opinion sampling, the population
    is divided into several disjoint groups $V_1, V_2, \ldots, V_m$ with $n_i = |V_i|$.
When we draw $y_i$ samples (with replacement) from group $V_i$ at time $t$, we obtain $y_i$ i.i.d.\ random variables
    $f_{v_{i,1}}^{(t)},f_{v_{i,2}}^{(t)},\dots, f_{v_{i,y_i}}^{(t)}$, 
	where $v_{i,j}$ is the $j$-th sample from group $V_i$.
	%for every $i\in [m]$ and $j\in [y_i]$.
Partitioned sampling uses 
    $\hat{f}^{(t)} = \sum_{i=1}^m \frac{n_i}{n}\cdot \left( \frac{1}{y_i} \sum_{j=1}^{y_i} f_{v_{i,j}}^{(t)} \right)$
    as the unbiased estimator for the mean population opinion at time $t$,
    and the task is to find the optimal allocation $(y_1^o, \dots, y_m^o)$ with sample size budget $\sum_{i=1}^m y_i^o \le k$ which minimizes the sample variance $\Var[\hat{f}^{(t)}]$, a common sample quality measure  \cite{bethel1986optimum,ballin2013joint,huang2017partitioned}.

One way to achieve best estimate quality for a future time $t$ is to do adaptive sampling 
	to quickly estimate the opinion variance of each group, and then decide the optimal sample size allocation
	for the real sample event at time $t$.
This corresponds to certain opinion polling practices, for instance,
	polling after each presidential debates, and preparing for a better sample quality at the election day.
We remark that in this setting, past samples are useful to estimate opinion variance within groups, but cannot be directly
	use to estimate the mean opinion at a future time $t$, since $\hat{f}^{(t)}$ is time-based and using 
	historical samples directly may lead to biased estimates.

%Prior sampling results can help improving the quality
%	of the samples at a future time $t$, but since $\hat{f}^{(t)}$
%	is time-based, prior sample results cannot be directly used 
%	to unbiasedly estimate the mean opinion at a future point.
%Instead, we could use these samples
%	to learn which groups have large variance and which
%	groups have small variance, so that at time $t$ we could assign
%	more samples to the groups with large variances to reduce the
%	overall variance $\Var[\hat{f}^{(t)}]$.
More specifically,
	let $X_i$ be the result of one random sample from group $V_i$
	in the steady state.
Note that the randomness of $X_i$ comes from both the sampling
	randomness and the opinion randomness in the steady state.
One can easily verify that 
	$\Var[\hat f^{(t)}]= \sum_{i=1}^{m} \frac{n_i^2}{n^2 y_i} \Var[X_i]$,
	where $\Var[X_i]$ is the variance of group $V_i$,
	and referred to as the {\em within-group variance}.
The goal is to use as few samples as possible to estimate within-group variances,
     and then give the final sample size allocation which minimizes $\Var[\hat f^{(t)}]$.

This falls into the CPE-CS framework.
In particular, 
    each group $V_i$ corresponds to an arm $i$,
    and each within-group variance $\Var[X_i]$ corresponds to the unknown parameter $\theta_i^*$ of arm $i$.
The decision space $\cY$ is $\{(y_1, \ldots, y_m) \in \mathbb{Z}_+^m \colon 
\sum_{i=1}^m y_i \le k \}$. %, where $k$ is the given sample size budget.
The reward function $r(\vec{\theta}; \vec{y})$ %= -h(\vec{\theta}; \vec{y}) =
is set to be $-\sum_{i=1}^{m} \frac{n_i^2 \theta_i}{n^2 y_i}$,
where the negative sign is because the partitioned opinion sampling problem is a minimization problem. 
It is non-linear but continuous and separable. 
Therefore, the problem is an instance of CPE-CS.
The oracle for the offline problem can be achieved by a greedy algorithm, denoted as $\phi(\vec\theta)$, and it satisfies the bi-monotonicity  (the design and the analysis of the offline oracle
is non-trivial, see \OnlyInFull{Appendix~\ref{app:osa}}\OnlyInShort{the extended version}).
Thus, the COCI algorithm can be directly applied as follows: 
1) $\alg{Est}_i$ is set to be the variance estimator, i.e., $\alg{Est}_i(X_{i,1}, \dots, X_{i,s})=\frac{1}{s-1} (\sum_{j=1}^{s}X_{i,j}^2-\frac{1}{s} (\sum_{j=1}^{s}X_{i,j})^2)$, 
and $\tau = 2$;
2)
the condition in line \ref{algline:oracle} of Algorithm~\ref{alg:general} is instantiated by
$\g_i(\underline{\vec{\theta}}_{-i,t-1}, \overline{\theta}_{i,t-1}) \ne \g_i(\overline{\vec{\theta}}_{-i,t-1}, \underline{\theta}_{i,t-1})$.

%\begin{corollary} \label{cor:CPESS}
%	We can instantiate the \alg{COCI} algorithm with the variance estimator % (Eq.\eqref{eq:varestimator})
%	 and the algorithmic implementation
%		as given by Eq.~\eqref{eq:SScondition} for the condition in line~\ref{algline:oracle} of Algorithm~\ref{alg:general}, and
%	the instantiation solves CPE-PS with the sample complexity as given in Eq.~\eqref{eq:samplecomplexity}.
%\end{corollary}

%% file: content/cpe-l.tex
\section{Applying COCI to CPE-L} \label{sec:appcpel}

%\wei{Includes (a) discussion on the equivalence of the stopping condition; (b) connection between
%	$\Lambda_i$ and $\Delta_{i}$ and the width}

In Section~\ref{sec:def}, we already show that the linear class CPE-L is a special case of CPE-CS.
In this section, we discuss the implication of applying COCI algorithm to solve CPE-L problems, and compare the
	sample complexity and implementation efficiency against
	the CLUCB algorithm in \cite{chen2014combinatorial}.
Since the parameter $\vec\theta^*$ is the vector of means of arms, we use the mean estimator and set $\tau=1$ in COCI.

Recall that for a binary vector $\vec{y} \in \cY$, $S_{\vec{y}}$ is defined as $\{i\in [m]\colon y_i = 1 \}$. 
\citet{chen2014combinatorial} use the term {\em reward gap} in the formulation of sample complexity.
For each arm $i\in [m]$, its {\em reward gap} $\Delta_i$ is defined as:
\begin{equation*}
%\label{eq:define-delta}
\Delta_i = \begin{cases}
r(\Realtheta; \vec{y}^*)- \max_{\vec{y}\in {\cY}, i\not\in S_{\vec y} } r(\Realtheta; \vec y), \text{if } i \in S_{\vec{y}^*}, \\
r(\Realtheta; \vec{y}^*)- \max_{\vec{y}\in {\cY}, i\in S_{\vec y} } r(\Realtheta; \vec y), \text{if } i \not\in S_{\vec{y}^*}.
\end{cases}
\end{equation*}

\citet{chen2014combinatorial} also define a (reward gap) hardness measure $\hardD = \sum_{i=1}^m \frac{1}{\Delta_i^2}$.
Moreover, for each decision class $\cY$, \citet{chen2014combinatorial} define a key quantity {\em width}, denoted as $\width(\cY)$, that is needed for sample complexity.
Intuitively, $\width(\cY)$ denotes the minimum number of elements that one may need to exchange in one step of a series of steps when changing the current decision $S\in \cY$ into another
decision $S'\in \cY$, and for every step of exchange in the series, the resulting decision (subset) should still be in $\cY$.
The technical definition is not very relevant with the discussion below, and thus is left in the supplementary material.
We remark that $\width(\cY) = O(m)$.

%\begin{equation}
%\label{eq:width-class}
%\width(\cY) = \min_{\cB \in \Exchange(\cY)} \width(\cB).
%\end{equation}

Given the above setup, \citet{chen2014combinatorial} show that with probability $1-\delta$, 
	their CLUCB algorithm achieves sample complexity bound
\begin{align}
T & \le 2m+ 499 \width(\cY)^2 \hardD \ln (4m \width(\cY)^2 \hardD/\delta) \nonumber \\
& = O\left(\width(\cY)^2 \hardD \log (m \hardD / \delta)\right). \label{eq:CLUCB-CPE-L}
\end{align}

When applying the COCI algorithm to solve CPE-L problems, we are able to obtain the following
	key connection between consistent optimality radius and the reward gap:
\begin{restatable}{lemma}{LambdaDelta} \label{lem:LambdaDelta}
	For the CPE-L problems, we have 
	$\forall i\in [m]$, $\Lambda_i \ge \Delta_i / \width(\cY)$, and thus $\hardL \le \hardD \cdot \width(\cY)^2$.
\end{restatable}
Combining with Theorem~\ref{thm:samplecomplexity}, we have that COCI could achieve the following sample complexity
	bound for CPE-L:
	\begin{align*} 
T &\leq 2m+12\width(\cY)^2 \hardD\ln (24\width(\cY)^2 \hardD) \\
	&\quad  +4\width(\cY)^2 \hardD\ln(4 \delta^{-1} ) \\
	&= O\left( \width(\cY)^2 \hardD \log ( m\hardD/\delta)  \right).%\label{eq:COCI-CPE-L}
\end{align*}

%\begin{corollary}
%Algorithm CLUCB in~\cite{chen2014combinatorial} is an instantiation of COCI algorithm for CPE-L, and based on the sample complexity of COCI (Theorem~\ref{thm:samplecomplexity}),
%CLUCB has sample complexity
%	\begin{align*} %\label{eq:CLUCBsamplecomplexity}
%	T\leq  & 2m  + 24 \width(\cY)^2 \hardD \ln \left( 12 \width(\cY)^2 \hardD \right)\\
%    &+ 4 \width(\cY)^2 \hardD 
%		\ln (4m\delta^{-1}).
%%	= O\left( \width(\cY) \hardD \ln\left( \frac{\width(\cY) \hardD m}{\delta}  \right)  \right).
%	\end{align*}
%\end{corollary}

The above result has the same sample complexity\footnote{CPE-L
    in~\cite{chen2014combinatorial} assumes $R$-sub-Gaussian distributions. Our analysis can be adapted to $R$-sub-Gaussian distributions as well, 
    with the same $R^2$ term appearing in the sample complexity.} as in Eq.~\eqref{eq:CLUCB-CPE-L} (with even a slightly better constant).
However, with our analysis, we only need the complicated combinatorial quantity $\width(\cY)$ and the linear reward assumption in the last step.
This also suggests that our consistent optimality radius $\Lambda_i$ and its associated consistent optimality hardness $\hardL$ are more fundamental
	measures of problem hardness than the reward gap $\Delta_i$ and its associated 
	reward gap hardness $\hardD$.
	
Next we discuss the implementation of the condition in line~\ref{algline:oracle} of COCI for CPE-L.
First, because linear functions are monotone, it is easy to see that we only need to check parameters $\vec\theta$
	on the boundaries of $\hat\Theta_{t-1}$ (at most $2|\cY|$ calls to the oracle $\g$).
For simple constraints such as any subsets of size $k$, it is easy to verify that $\g(\vec\theta)$ is bi-monotone
	in this case, and thus we have efficient implementation of the condition as given in Theorem~\ref{thm:condition}.
For more complicated combinatorial constraints, it is still an open question on whether efficient implementation
	of the condition in line~\ref{algline:oracle} exists when oracle $\g$ is given.
The CLUCB algorithm, on the other hand, does have an efficient implementation for all CPE-L problems as long as
	the oracle $\g$ is given.
	
Therefore, compared with CLUCB in terms of efficient implementation, COCI can be viewed as taking the trade-off between
	the complexity of the reward functions and the complexity of combinatorial constraints.
In particular, COCI could handle more complicated nonlinear reward functions on real vectors, 
	and allow
	efficient implementation (due to bi-monotonicity) under simple constraints, while CLUCB deals with complicated
	combinatorial constraints but could only work with linear reward functions on binary vectors.

%% file: content/conclusion.tex
\section {Future Work}

There are a number of open problems and future directions.
For example, one can consider the fixed budget setting of CPE-CS: the game stops after a fixed number $T$ of rounds where $T$ is given before the game starts, and the learner needs to minimize the probability of error  $\Pr[\vec y^o\ne \vec y^*]$.
One may also consider the PAC setting: with probability at least $1-\delta$ the algorithm should output
	a decision with reward at most $\varepsilon$ away from the optimal reward.
This setting may further help to eliminate the requirement of finite decision class $\cY$.
Another direction is to combine the advantage of COCI and CLUCB to design a unified algorithm that allows efficient
	implementation for all CPE-CS problems. 
How to incorporate approximation oracle instead of the exact
	oracle into the CPE framework is also an interesting direction.

%% file: content/appendix.tex
% !TEX root =  ../ijcai18.tex

\clearpage
\appendix
\onecolumn
\setlength\parindent{0pt}
 \allowdisplaybreaks

\section*{\centering \LARGE Appendix}

\bigbreak\bigbreak

We give all the proofs of lemmas, theorems and extra discussions in the appendix organized by sections, 
i.e., Section~\ref{app:cpecs} for solving CPE-CS, Section~\ref{app:application} for applications, and Section~\ref{app:applycpe-l} for applying COCI to CPE-L.

\section{Proofs for Section~\ref{sec:cpecs}: Solving CPE-CS}\label{app:cpecs}
In Section~\ref{app:solvecpecs}, we give the proofs for the properties of $\Lambda_i$. 
Then we show the proofs of the upper and lower bounds of the sample complexity in Section~\ref{app:upperbound} and \ref{app:lowerbound},
    and proof of Theorem~\ref{thm:condition} in Section~\ref{app:condition}.
Moreover, we give a discussion on sampling complexity of the uniform sampling in Section~\ref{app:uni}.

\subsection{Proofs for the Properties Of $\Lambda_i$}\label{app:solvecpecs}
\LambdaProposition*
\begin{proof}
    According to the definition of $\Lambda_i=\inf_{\vec\theta: \g_i(\vec\theta) \neq \g_i(\Realtheta)} \left\| \vec\theta-\Realtheta \right\|_\infty= \inf_{\vec\theta: \g_i(\vec\theta) \neq \g_i(\Realtheta)} \max_{j\in[m] } \abs{\theta_j-\realtheta_j}$,
    $\forall \vec\theta \in [0,1]^m$, if $\g_i(\vec\theta)\neq \g_i(\Realtheta)$, there exists $j\in[m]$ such that $\abs{\theta_j-\realtheta_j} \geq \Lambda_i$,
    which indicates that if $\abs{\theta_j-\realtheta_j}<\Lambda_i$ holds for all $j\in[m]$, then $\g_i(\vec\theta)=\g_i(\Realtheta)$.
\end{proof}

\LambdaPositive*
\begin{proof}
    It is straightforward to see that $\Lambda_i$ is non-negative for any $i\in[m]$.
    Suppose, for a contradiction, that there exists some $i$ such that $\Lambda_i=\inf_{\vec\theta: \g_i(\vec\theta) \neq \g_i(\Realtheta)} \left\| \vec\theta-\Realtheta \right\|_\infty=0$.
    Thus
    $\forall \varepsilon>0$, $\exists \vec{\theta}$, $\norm{ \vec\theta-\Realtheta }_\infty<\varepsilon$ and $\vec y=\g(\vec \theta)$ where $\vec y$'s $i$-th component $y_i\neq y_i^*$.
    Since $\g$ returns an optimal decision, we have $r(\vec\theta;\vec y) \geq r(\vec\theta;\vec y^*)$.
    
    Therefore, for any infinite positive sequence: $\{\varepsilon^{(n)}\}_{n=1}^\infty$ where $\lim_{n\rightarrow \infty}\varepsilon^{(n)}=0$,
    there exist sequences $\{\vec{\theta}^{(n)}\}_{n=1}^\infty$ and $\{\vec{y}^{(n)}\}_{n=1}^\infty$ 
    such that $\norm{ \vec\theta^{(n)}-\Realtheta}_\infty<\varepsilon^{(n)}$, 
    $\vec y^{(n)} = \g(\vec\theta^{(n)})\neq \vec y^*$ and $r(\vec\theta^{(n)};\vec y^{(n)}) \geq r(\vec\theta^{(n)};\vec y^*)$.
    Notice that the decision class $\mathcal{Y}$ is finite, thus there exists some $\vec{ \tilde y}$ that occurs an infinite number of times in sequence $\{\vec{y}^{(n)}\}_{n=1}^\infty$.
    Let $\{\vec{y}^{(s_k)}\}_{k=1}^\infty$ be the subsequence with $1\le s_1 < s_2 < s_3 < \cdots$,
    such that $\vec{y}^{(s_k)}=\vec{ \tilde y}$
    for all $k\ge 1$.
    %    We take down the order number $s_1, s_2, \dots, s_k, \dots$ when $\vec y^{(n)}=\vec{ \tilde y}$.
    
    Since $r(\vec\theta; \vec{ \tilde y})$ and  $r(\vec\theta; \vec{y}^*)$
    are both continuous functions with respect to $\vec{\theta}$, we have
    \[
    r(\Realtheta;\vec{ \tilde y}) =\lim_{k\rightarrow \infty} r(\vec\theta^{(s_k)};\vec{ \tilde y}) 
    \geq
    \lim_{k\rightarrow \infty} r(\vec\theta^{(s_k)};\vec y^*)= r(\Realtheta;\vec y^*).
    \]
    Thus $\vec{ \tilde y} \ne \vec y^*$ is also an optimal solution with input $\Realtheta$, which contradicts to the prerequisite that the optimal solution $\vec{y}^*$
    is unique.
    Therefore, $\forall i\in[m]$, $\Lambda_i>0$.
\end{proof}

\subsection{Proof of Theorem~\ref{thm:samplecomplexity}}\label{app:upperbound}
We first prove the following lemma, which will be used in the proof of Theorem~\ref{thm:samplecomplexity}.
\begin{restatable}{lemma}{highprobability}\label{lem:xi probability}
    Event $\xi$ occurs with probability at least $1-\delta$.
\end{restatable}

The proof of the above lemma uses the following concentration result.
\begin{fact}[McDiarmid's Inequality \cite{mcdiarmid1989method}] \label{fact:McDiarmid}
    Let $X_1, X_2, \dots, X_n$ be independent random variables taking values from the set $\mathcal{X}$, 
    and $f: \mathcal{X}^n \rightarrow \mathbb{R}$ be a function of $X_1, X_2, \dots, X_n$ which satisfies
    \[
    \sup_{x_1, x_2, \dots, x_n, x_i'\in \mathcal{X}}\Abs{f(x_1, \dots, x_i, \dots, x_n)-f(x_1, \dots, x_i', \dots, x_n)} \leq c_i, \forall i\in[n].
    \]
    Then for any $\varepsilon>0$,
    \[
    \Pr\left[ \Abs{f(X_1, X_2, \dots, X_n) - \E[f(X_1, X_2, \dots, X_n)]} \geq \varepsilon \right] 
    \leq 2\exp \left(- \frac{2\varepsilon^2}{\sum_{i=1}^n c_i^2}\right). 
    \]
\end{fact}

%\wei{Note that the following may not apply to the $R$-subgaussian case, but we can use a relaxed 
%	inequality, so the overall result should still work, but the constant is a bit larger.}

\begin{proof}[Proof of Lemma~\ref{lem:xi probability}]
    Let $\xi_{i,t}$ denote the event that $|\hat{\theta}_{i,t} - \theta^*_i| \le \rad_{i,t}$ and
    $\xi_t := \bigcap_{i \in [m]} \xi_{i,t}$. Then
    \begin{align*}
    \Pr [\xi]
    =\Pr\left[\bigcap_{t=\tau m}^\infty \xi_t\right].
    \end{align*}
   Recall that 
   $\alg{Est}_i(X_{i,1}, X_{i,2},\ldots, X_{i,s}) = \frac{1}{s} \sum_{j=1}^s X_{i,j}$ for the mean estimator, and
   $\alg{Est}_i(X_{i,1}, \dots, X_{i,s})=\frac{1}{s-1} \left(\sum_{j=1}^{s}X_{i,j}^2-\frac{1}{s}(\sum_{j=1}^{s}X_{i,j})^2\right)$ for the variance estimator.
   Then both of the estimators satisfy that
   \[
   \sup_{x_1, x_2, \dots, x_s, x_j'\in [0,1]}\Abs{\alg{Est}_i(x_1, \dots, x_j, \dots, x_s)-\alg{Est}_i(x_1, \dots, x_j', \dots, x_s)} \leq \frac{1}{s}, \forall j\in[s].
   \]
   According to the McDiarmid's inequality (Fact~\ref{fact:McDiarmid}),     
   \[
   \Pr\left[\Abs{\hat\theta_{i}^{(s)}-\realtheta_i} \geq \varepsilon \right] 
   \leq 2\exp\left(-\frac{2 \varepsilon^2}{\sum_{i=1}^s 1/s^2} \right)
   = 2\exp(-2 s \varepsilon^2),
   \]
   where $\hat\theta_{i}^{(s)}$ is the estimate of $\theta_i^*$ using $s$ samples.
   Therefore, $\forall \varepsilon>0$, $\forall i\in [m]$, $\forall t\geq \tau m$,
    %    \[
    %    \Pr\left[\Abs{\hat\theta_{i,t}-\realtheta_i} \geq \varepsilon\right] \leq 2\exp(-2 T_{i,t} \varepsilon^2).
    %    \]
    %    Thus
    \begin{align*}
    \Pr[\neg \xi_{i,t}]=\Pr\left[ \Abs{\hat{\theta}_{i,t}-\realtheta_i} \geq \rad_{i,t} \right]
    &=\Pr\left[ \Abs{\hat{\theta}_{i,t}-\realtheta_i} \geq \sqrt{\frac{1}{2T_{i,t}}\ln\frac{4t^3}{\tau\delta}} \right]\\
    &= \sum_{s=1}^{t}\Pr\left[ \Abs{\hat{\theta}_i^{(s)}-\realtheta_i} \geq \sqrt{\frac{1}{2s}\ln\frac{4t^3}{\tau\delta}}, T_{i,t}=s \right] \\
    &\le \sum_{s=1}^{t}\Pr\left[ \Abs{\hat{\theta}_i^{(s)}-\realtheta_i} \geq \sqrt{\frac{1}{2s}\ln\frac{4t^3}{\tau\delta}} \right] \\
    &\leq \sum_{s=1}^{t} \frac{\tau\delta}{2t^3} = \frac{\tau\delta}{2t^2}.
    \end{align*}
    By a union bound over all $i\in[m]$, we see that $\Pr[\xi_t]\geq 1-\sum_{i=1}^m \Pr[\neg \xi_{i,t}] \ge 1- \frac{m \tau  \delta}{2t^2}$. 
    Using a union bound again over all $t > 0$, we have
    \begin{align*}
    \Pr [\xi]
    =\Pr\left[\bigcap_{t=\tau m}^\infty \xi_t\right]
    \geq 1-\sum_{t=\tau m}^{\infty}\Pr[\neg \xi_t]
    \geq 1-\sum_{t=\tau m}^\infty \frac{m\tau\delta}{2t^2}
    \geq 1-\frac{m\tau\delta}{2}\sum_{t=m\tau}^\infty \frac{1}{t^2-\frac{1}{4}}
    \geq 1-\frac{m\tau\delta}{2}\frac{1}{(\tau m)-\frac{1}{2}}
    \geq 1-\delta.
    \end{align*}
    Therefore,  with probability at least $1-\delta$, event $\xi$ occurs.
\end{proof}

\samplecomplexity*
The proof of Theorem~\ref{thm:samplecomplexity} uses the following result.
\begin{fact}[\cite{shalev2014understanding}]\label{fact:log}
    Let $a \geq 1$ and $b > 0$. Then: $x \leq a \log(x)+b \Rightarrow x \leq 4a \log(2a)+2b$.
\end{fact}
\begin{proof}[Proof of Theorem~\ref{thm:samplecomplexity}]
    Suppose event $\xi$ occurs.
    The output of the algorithm must be the optimal $\vec{y}^*$.
    This is because when the COCI algorithm stops in round $t$, $\Realtheta\in \hat \Theta_{t-1}$ and 
    $\max_{\vec\theta\in\hat\Theta_{t-1}} \g_i(\vec\theta) = \min_{\vec\theta\in\hat\Theta_{t-1}} \g_i(\vec\theta)$ for all $i\in [m]$.
    Thus the output $\vec{y}^o = \g(\vec\theta) = \g(\Realtheta) = \vec{y}^*$, for any $\vec\theta \in \hat\Theta_{t-1}$.
    
    We now prove the upper bound of the sample complexity $T$.
    Let $T_i$ be the total number of times that arm $i$ is played, and $t_i$ be the last round that arm $i$ is played.
    If $t_i> \tau m$, according to Lemma~\ref{lem:radius}, 
    \[\frac{\Lambda_i}{2} \leq \sqrt{\frac{1}{2(T_i-1)}\ln\frac{4t_i^3}{\tau  \delta}} \leq \sqrt{\frac{1}{2(T_i-1)}\ln\frac{4T^3}{\tau  \delta}}.  \]
    Thus,
    \begin{equation}\label{eq:Ti}
    T_i \leq 1+\frac{2}{\Lambda_i^2}\ln\frac{4T^3}{\tau  \delta}.
    \end{equation}
    If $t_i\leq \tau m$, then $T_i \leq \tau\le 2$. Since $\Lambda_i\leq 1$, $\delta<1$, $T\geq \tau$, Equation (\ref{eq:Ti}) also holds.
    Therefore,
    \[ T=\sum_{i=1}^m T_i \leq m+2\left( \sum_{i=1}^m \frac{1}{\Lambda_i^2}\right) \ln\frac{4T^3}{\tau  \delta}=m+2\hardL \ln\frac{4T^3}{\tau  \delta}=6\hardL \ln T + m+2\hardL \ln \frac{4}{\tau  \delta}.  \]
    Plugging $\frac{T}{6\hardL}$, $1$, 
    $\ln 6\hardL +\frac{1}{6\hardL} (m + 2\hardL \ln \frac{4}{\tau \delta } ) >0 $ in $x$, $a$, $b$ of Fact~\ref{fact:log},
    it follows that
    \[ T \leq 2m  + 4\hardL \left(3\ln \hardL +11 \ln 2+3\ln 3-\ln \tau\delta\right)=O\left(\hardL\log \frac{\hardL}{\delta}\right). \]
    
    Finally, according to Lemma~\ref{lem:xi probability},  event $\xi$ occurs with probability at least $1-\delta$. Thus with probability at least $1-\delta$, 
    the COCI algorithm returns the optimal solution, and the upper bound of sample complexity given in the theorem holds.
\end{proof}

\subsection{Proof of Theorem~\ref{thm:lowerbound}}\label{app:lowerbound}
\lowerbound*
\begin{proof}
    \citet{ChenGLQW17} in Lemma C.1 show that given an integer $m$, there exist a universal constant $c$ and a list of subsets $S_1$, $S_2$, \dots, $S_n$ of $[m]$ with $n=2^{cm}$, such that $|S_i|=l=\Omega(m)$ for each $S_i$, and $|S_i\cap S_j|\le l/2$ for each $i\ne j$.
    We fix a real number $\Delta\in(0,0.1)$, and let constant $c$, $l=\Omega(m)$, and $S_1$, $S_2$, \dots, $S_n$ be as the above fact.

    We define an $m$-arm CPE-L instance $\mathcal{C}_{S_1}$ whose $i$-th arm has mean $\Delta$ when $i\in S_1$ and mean 0 otherwise.
    We set the decision class to be $\mathcal{S}=\{S_1, S_2, \dots, S_n\}$.
    Let $\xi_S$ be the event that algorithm $\mathcal{A}$ outputs $S$, thus we have
    $\sum_{S\in\mathcal{S},S\ne S_1}\Pr[\xi_S] < \delta$.
    By a simple averaging argument, there exists some decision $S'\in \mathcal{S}\setminus \{S_1\}$ such that
    $\Pr[\xi_{S'}]<\delta/(|\mathcal{S}|-1)\le 2\delta/|\mathcal{S}|$.
    
    We define another $m$-arm CPE-L instance $\mathcal{C}_{S'}$ whose $i$-th arm has mean $\Delta$ when $i\in S'$ and mean 0 otherwise.
    Then $\Pr_{S'}[\xi_{S'}]\ge 1-\delta>0.9$ where the subscript $S'$ of $\Pr$ represents the instance.
    According to Lemma 2.3 in \cite{ChenGLQW17}, algorithm $\mathcal{A}$ must spend at least
    \begin{equation}
    d\left(\Pr_{S'}[\xi_{S'}], \Pr_{S_1}[\xi_{S'}]\right) \cdot \Delta^{-2}=\Omega\left((m+\ln \delta^{-1})\cdot\Delta^{-2}\right)\label{eq:lower bound}
    \end{equation}
    samples on instance $\mathcal{C}_{S'}$ in expectation.
    
    On the other hand, on instance $\mathcal{C}_{S'}$, we have
    $$\hardL=\sum_{i=1}^m \frac{1}{\Lambda_i}\le4m\cdot\Delta^{-2}.$$
    One can easily verify that the lower bound of CPE-CS $\Omega(\hardL + \hardL\log\delta^{-1}/m)$ matches Eq. (\ref{eq:lower bound}).
\end{proof}

\subsection{Proof of Theorem~\ref{thm:condition}}\label{app:condition}
{\lucb*}
\begin{proof}
    
    $\forall t > \tau m$, $\forall \vec \theta \in \hat \Theta_{t-1}$, $\forall i\in[m]$,  $\theta_{i} \in[ \underline{\theta}_{i,t-1}, \overline{\theta}_{i,t-1}]$.
    
    If $\phi(\vec\theta)$ is non-decreasing monotone in $\theta_i$,
    we have $\g_i(\vec \theta)\leq \g_i(\vec \theta_{-i}, \overline{\theta}_{i,t-1})$.
    According to the bi-monotonicity, $\phi(\vec\theta)$ is non-increasing monotone in $\theta_j$ for each $j\neq i$,
    thus we have $\g_i(\vec \theta_{-i}, \overline{\theta}_{i,t-1}) \leq \g_i(\underline{\vec \theta}_{-i,t-1}, \overline{\theta}_{i,t-1})$.
    This means that $\g_i(\vec \theta)\leq \g_i(\underline{\vec \theta}_{-i,t-1}, \overline{\theta}_{i,t-1})$ holds for any $\vec \theta \in \hat \Theta_{t-1}$.
    In addition, since $(\underline{\vec \theta}_{-i,t-1}, \overline{\theta}_{i,t-1})\in \hat\Theta_{t-1}$,
    therefore $\max_{\vec\theta\in\hat\Theta_{t-1}} \g_i(\vec\theta) = \g_i(\underline{\vec{\theta}}_{-i,t-1}, \overline{\theta}_{i,t-1})$.
    Symmetrically, one can prove that $\min_{\vec\theta\in\hat\Theta_{t-1}} \g_i(\vec\theta) = \g_i(\overline{\vec{\theta}}_{-i,t-1}, \underline{\theta}_{i,t-1})$.
    
    Similarly,  if $\phi(\vec\theta)$ is non-increasing monotone in $\theta_i$, one can prove that
    $\min_{\vec\theta\in\hat\Theta_{t-1}} \g_i(\vec\theta) = \g_i(\underline{\vec{\theta}}_{-i,t-1}, \overline{\theta}_{i,t-1})$ and
    $\max_{\vec\theta\in\hat\Theta_{t-1}} \g_i(\vec\theta) = \g_i(\overline{\vec{\theta}}_{-i,t-1}, \underline{\theta}_{i,t-1})$.
    
    Therefore, $\max_{\vec\theta\in\hat\Theta_{t-1}} \g_i(\vec\theta) \ne \min_{\vec\theta\in\hat\Theta_{t-1}} \g_i(\vec\theta)$ is equivalent to $\g_i(\underline{\vec{\theta}}_{-i,t-1}, \overline{\theta}_{i,t-1}) \ne \g_i(\overline{\vec{\theta}}_{-i,t-1}, \underline{\theta}_{i,t-1})$.
\end{proof}

\subsection{Uniform Sampling}\label{app:uni}

If we change the line~\ref{line:largestradius} in Algorithm~\ref{alg:general} to $$j \leftarrow \argmax_{i\in [m]}\rad_{i,t-1},$$ which 
selects the arm with the largest confidence radius among all arms, then it turns to the uniform sampling.
In particular, at round $\tau m$, all arms have the same radius.
Then the algorithm will pick each arm once in the next $m$ rounds.
At round $\tau m+m$, all arms have the same radius again.
Thus we say that it is the uniform sampling.
The uniform version of Lemma \ref{lem:radius} is as follows.
\begin{lemma} \label{lem:uniformradius}
    Suppose event $\xi$ occurs. For every $i\in[m]$ and every $t>\tau m$, if $\rad_{i,t-1}<\min_i \Lambda_i/2$, then arm $i$ will not be played in round $t$ for the uniform sampling.
\end{lemma}

\begin{proof}
    Suppose, for a contradiction, that arm $i$ is played in round $t$, namely,
    $i=\argmax_{j\in [m]}\rad_{j,t-1}$.
    Thus for each $j\in [m]$, we have $\rad_{j,t-1} \leq \rad_{i,t-1} <\min_{l\in[m]} \Lambda_l/2$.
    
    Since event $\xi$ occurs, for each $j\in [m]$, we have $\hat \theta_{j,t-1}\in[\realtheta_j- \rad_{j,t-1},\ \realtheta_j+ \rad_{j,t-1}]$,
    and thus
    \begin{align*}
    \left[\hat \theta_{j,t-1}-\rad_{j,t-1},\ \hat \theta_{j,t-1}+\rad_{j,t-1}\right] 
    &\subseteq \left[\realtheta_j-2\rad_{j,t-1},\ \realtheta_j+2\rad_{j,t-1}\right]
    \subset \left(\realtheta_j-\min_{l\in[m]} \Lambda_l,\ \realtheta_j+\min_{l\in[m]} \Lambda_l\right).
    \end{align*}
   
    $\forall \vec\theta\in\hat{\Theta}_{t-1}\subset \bigotimes_{j=1}^m  \left[\hat \theta_{j,t-1}-\rad_{j,t-1},\ \hat \theta_{j,t-1}+\rad_{j,t-1}\right]$, 
    $\abs{\theta_j-\realtheta_j}<\min_{l\in[m]} \Lambda_{l} \leq \Lambda_i$ holds for all $j\in[m]$.

    According to Proposition~\ref{proposition:Lambda}, $\g_i(\vec {\theta})=\g_i(\Realtheta)$ holds for all $\vec\theta\in\hat{\Theta}_{t-1}$, and thus $i\notin C_t$,
    which means arm $i$ will not be played in round $t$.
\end{proof}

We define another hardness measurement
$\hardU = \frac{m}{\min_{i\in[m]}\Lambda_i^2}$.
With Lemma~\ref{lem:uniformradius} and applying the similar analysis as the proof of Theorem~\ref{thm:samplecomplexity}, one can 
obtain  the sample complexity of uniformly sampling $T^{\rm uniform}$ as follows.

\begin{theorem}
    With probability at least $1-\delta$, the COCI algorithm with uniform sampling
    returns the unique true optimal solution $\vec{y}^o = \vec{y}^*$, and the number of rounds (or samples) $T^{\rm uniform}$ satisfies
    \begin{align*}
    T^{\rm uniform} \leq 2m+12\hardU\ln 24\hardU +4\hardU\ln\frac{4}{\tau\delta } 
    = O\left( \hardU \log \frac{\hardU} {\delta }  \right),
    \end{align*}
    and the factor $\hardU$ is necessary in the above sample complexity bound.
\end{theorem}

\begin{proof}
    One can easily verify the upper bound of sample complexity by Lemma~\ref{lem:uniformradius} and applying the similar analysis as the proof of Theorem~\ref{thm:samplecomplexity}.
    We only prove that the factor $\hardU$ is necessary in the sample complexity bound for the uniform sampling.

There is a well known fact that to distinguish whether a Bernoulli random variable has mean $1/2$ or
$1/2 +\varepsilon$  requires at least $\Omega(\frac{1}{\varepsilon^2})$ samples
\cite{chernoff1972sequential,anthony2009neural}.
Precisely,
    fix $\varepsilon\in(0,0.02)$ and let $X$ be a Bernoulli random variable with mean being either $1/2$ or $1/2+2\varepsilon$. If an algorithm $\mathcal{A}$ can output the correct mean of $X$ with probability at least 0.51, then the expected number of samples performed by $\mathcal{A}$ is at least $\Omega(\frac{1}{\varepsilon^2})$ (Lemma 5.3 in \cite{zhou2014optimal}).

Thus, we can construct a set of $m$ Bernoulli arms, where the first arm has a known
mean $1/2 + \varepsilon$, and
the second arm has mean either $1/2$ or $1/2 +2 \varepsilon$, and the rest arms have mean
$\varepsilon$.
Then for any pure exploration bandit algorithm to identify the best arm, it at least needs to distinguish
whether arm $1$ or arm $2$ is better, and thus it must take 
$\Omega(\frac{1}{\varepsilon^2})$ samples of
arm $2$.
Since the sampling is uniform, all the arms will be played the same number of times even if the rest arms perform pretty bad, 
thus with probability at least 0.51, the expected number of samples is at least $\Omega(\frac{m}{\varepsilon^2})$.
Notice that in this example, $\Lambda_1 = \Lambda_2 = \varepsilon/2$, and 
$\Lambda_3 = \cdots = \Lambda_m = 1/4$ or $1/4 + \varepsilon/2$, therefore, the expected sample complexity
is indeed $\Omega(\hardU) = \Omega(\frac{m}{\min \Lambda_i^2})$, and it is significantly larger than
$\hardL = \sum_{i\in [m]} \frac{1}{\Lambda_i^2} = O(\frac{1}{\varepsilon^2}+m)$, which is the 
key factor bounding the sample complexity of the adaptive COCI algorithm.
\end{proof}

\section{Proofs for Section~\ref{sec:application}: Applications}\label{app:application}

\subsection{Proofs for Section~\ref{sec:water}: Water Resource Planning}
\waterbi*
\begin{proof}
    The gradient of the reward function is
    \begin{align*}
    \nabla r(\vec\theta, \vec y)=\left(\theta_1-\frac{\d f_1}{\d y_1}(y_1), \dots, \theta_m-\frac{\d f_m}{\d y_m}(y_m)\right).
    \end{align*} 
    For each $\vec\theta$, since the constraint $\sum_{i=1}^m y_i \ge b$ is tight, the gradient at the optimal point $\nabla r(\vec\theta, \phi(\vec\theta))$ should be parallel to the normal vector of the plane $\sum_{i=1}^m y_i =b$, one of which is $(1,1, \dots, 1)$.
    Thus there exists some $\lambda$ such that $$\theta_i-\frac{\d f_i}{\d y_i}(\phi_i(\vec\theta))=\lambda, \forall i\in[m].$$
    
    When some $\theta_i$ varies  $\delta\theta_i$, namely $\vec\theta'=(\theta_1, \dots, \theta_i+\delta\theta_i,\dots, \theta_m)$, there exists some $\lambda'$ such that
    \begin{align*}
    \theta_i+\delta\theta_i-\frac{\d f_i}{\d y_i}(\phi_i(\vec\theta'))&=\lambda',\\
    \theta_j-\frac{\d f_j}{\d y_j}(\phi_j(\vec\theta'))&=\lambda', \forall j\ne i.
    \end{align*}
    Thus
    \begin{align*}
    \frac{\d f_j}{\d y_j}(\phi_j(\vec\theta'))-\frac{\d f_j}{\d y_j}(\phi_j(\vec\theta))=\lambda-\lambda', \forall j\ne i.
    \end{align*}
    
    Without loss of generality, we assume that \{$\frac{\d f_i}{\d y_i}$\}'s are all monotonically decreasing.
    
    If $\lambda'\le \lambda$, then for each $j\ne i$, $\frac{\d f_j}{\d y_j}(\g_j(\vec\theta')) \ge\frac{\d f_j}{\d y_j}(\g_j(\vec\theta))$, which indicates that $\g_j(\vec\theta')\le \g_j(\vec\theta)$.
    Since $\sum_{i=1}^m y_i \geq b$ is tight, we have $\g_i(\vec\theta')\ge \g_i(\vec\theta)$.
    Similarly, if $\lambda'\ge \lambda$, we have $\g_j(\vec\theta')\ge \g_j(\vec\theta)$ for every $j\ne i$ and $\g_i(\vec\theta')\le \g_i(\vec\theta)$.
    This indicates that $\g_i(\vec\theta)$ and $\g_j(\vec\theta)$ ($\forall j \ne i$) alway vary oppositely.

    The rest part we need to show is that $\g_i(\vec\theta)$ is strictly monotone.
    If it isn't, then there exists $\delta\theta_i'\ne 0$ and let $\vec\theta''=(\theta_1, \dots, \theta_i+\delta\theta_i',\dots, \theta_m)$, such that $\g_i(\vec\theta)=\g_i(\vec\theta'')$.
    Thus there exists some $\lambda''$ such that
    \begin{align*}
    \theta_i+\delta\theta_i'-\frac{\d f_i}{\d y_i}(\phi_i(\vec\theta''))&=\lambda'',\\
    \theta_j-\frac{\d f_j}{\d y_j}(\phi_j(\vec\theta''))&=\lambda'', \forall j\ne i.
    \end{align*}
    If $\delta\theta_i'>0$, since $\g_i(\vec\theta)=\g_i(\vec\theta'')$, then $\lambda''>\lambda$.
    Thus for all $j\ne i$, $\frac{\d f_j}{\d y_j}(\phi_j(\vec\theta''))<\frac{\d f_j}{\d y_j}(\phi_j(\vec\theta))$ which indicates that $\g_j(\vec\theta'')>\phi_j(\vec\theta)$.
    Then the constraint $\sum_{i=1}^m y_i \geq b$ is not tight any longer.
    If $\delta\theta_i'<0$, then $\lambda''<\lambda$.
    Thus for all $j\ne i$, $\frac{\d f_j}{\d y_j}(\phi_j(\vec\theta''))>\frac{\d f_j}{\d y_j}(\phi_j(\vec\theta))$ which indicates that $\g_j(\vec\theta'')<\phi_j(\vec\theta)$.
    Then the constraint $\sum_{i=1}^m y_i \geq b$ is violated.
    Therefore $\delta\theta_i'=0$ and $\g_i(\vec\theta)$ is monotone.
\end{proof}

\subsection{Proofs for Section~\ref{sec:appcpess}: Partitioned Opinion Sampling}\label{app:osa}
In this section, we first give the definition of the offline problem of partitioned opinion sampling (Definition~\ref{def:osa}), 
    and then propose a greedy algorithm (Algorithm~\ref{alg:greedy}) to solve it.
We analyze the greedy algorithm and show that it outputs the leading optimal solution in Theorem~\ref{thm:OSA}.
Finally, we show that the leading optimal solution satisfies the bi-monotonicity in Lemma~\ref{lem:bimonotonicity}.

\begin{definition}[Optimal Sample Allocation (OSA)]\label{def:osa}
    Given (a) parameter $\vec{\theta}=(\theta_1, \theta_2, \dots, \theta_m) \in [0,1]^m$ where $\theta_i=\Var\left[X_i\right]$, 
    (b) $n_1, n_2, \ldots, n_m$ where $n_i = |V_i|$, and
    (c) a positive integer $k\in \mathbb{Z}_+$ as the sample size budget,
    the {\em optimal sample allocation (OSA)} problem is to find
    an optimal allocation 	$\vec y^{g} = (y_1^g, y_2^g, \dots, y_m^g)\in \mathbb{Z}_+^m$ such that 
    $\sum_{i=1}^m y_i^g \leq k$ and the total variance
    $\Var \left[\hat f^{(t)}\right] = \sum_{i=1}^{m} \frac{n_i^2 \theta_i}{n^2 y_i^g} $ is minimized,
    i.e., $\vec{y}^g \in \argmin_{\vec{y}\in \mathbb{Z}_+^m \land \norm{\vec{y}}_1 \le k} \sum_{i=1}^{m} \frac{n_i^2 \theta_i}{n^2 y_i} $.
\end{definition}

In the OSA problem, the factor of $1/n^2$ is a constant and immaterial for the optimization task, 
and henceforth we remove this factor from our discussion.
Let $h(\vec{\theta}; \vec{y}) = \sum_{i=1}^{m} \frac{n_i^2 \theta_i}{y_i} $ be the objective 
function. 
By the Cauchy-Schwartz inequality, we know that the optimal real-valued solution for minimizing 
$h(\vec{\theta}; \vec{y})$ is $y_i = \frac{n_i \sqrt{\theta_i}}{\sum_{j=1}^m n_j \sqrt{\theta_j}} \cdot k=Z n_i \sqrt{\theta_i}\cdot k$,
where $Z$ denotes the normalization factor for convenience.
Prior studies (e.g., \cite{CarpentierM11,etore2010adaptive}), either 
stop at this point or only consider simple rounding to the closest
integer solution.
However, we want to find the exact optimal {\em integral} solution as required by the OSA problem.
It is not trivial to transfer from a real-valued solution to an integral solution, because (a) even only
considering simple rounding, there are an exponential number of options for rounding up or down
each real value	to maintain the sample budget constraint, and (b) only rounding to
the closest integer may be quite away from the optimal solution.

\begin{algorithm}[hbt]
    %\SetKwComment{tcp}{\tiny // }{}%
    \caption{\alg{GreedyOSA}: Greedy algorithm for the offline OSA problem} \label{alg:greedy}
    \KwIn{$m$, $\vec{\theta} \in [0,1]^m$, $n_1, n_2, \ldots, n_m$,
        integer $k\geq m$}
    \KwOut{$\vec y^g \in \mathbb{Z}_+^m$}
    
    $Z \leftarrow 1 / \sum_{j=1}^m n_j \sqrt{\theta_j}$\;
    \For{$i = 1, 2, \dots, m $}{ \label{line:deltab}
        \eIf{$\left\lceil Z  n_i \sqrt{\theta_i} \cdot k \right\rceil 
            \left (\left\lceil Z  n_i \sqrt{\theta_i} \cdot k \right\rceil - 1\right)\ge (Z  n_i \sqrt{\theta_i} \cdot k)^2$}{
            $\delta_i \leftarrow 0$\;
        }{
            $\delta_i \leftarrow \left\lceil Z  n_i \sqrt{\theta_i} \cdot k \right\rceil - Z  n_i \sqrt{\theta_i} \cdot k $\; %\tcp*{$\delta_i$ is the slack that could be used to push down the base for other dimensions}
            \label{line:deltae}
        }
    }
    \For{$i = 1, 2, \dots, m $}{ 
        $y^{(0)}_i \leftarrow \max(1, \left\lfloor Z n_i \sqrt{\theta_i} \cdot k - \sum_{j\ne i} \delta_j \right\rfloor)$\; \label{line:initialization} %\tcp*{initialization the base vector $\vec y^{(0)}$ }
    }
    \eIf{ $\sum_{i=1}^m y_i^{(0)} = k$}{
        $\vec{y}^g \leftarrow \vec y^{(0)}$\;
    }{
        \For{$t = 1, 2, \dots$ \label{line:greedyb}}{
            $C^{(t)} \leftarrow
            \argmax_{i\in [m]} \left( \frac{n_i^2 \theta_i}{y^{(t-1)}_i} - \frac{n_i^2 \theta_i}{ (y^{(t-1)}_i+1)}\right) $\;  \label{line:max} 
            %        $C^{(t)} \leftarrow % \argmax_{i \in [m]} r(\vec \theta; \vec y^{(t-1)}+\vec e_i) =
            %	        \argmin_{i\in [m]} \sum_{j=1}^{m} \frac{n_j^2 \theta_j}{ (y_j+e_{i,j})}$
            %	        \tcp*{$e_{i,i}=1$, $e_{i,j}=0$ for $j\ne i$, $\vec e_i = (e_{i,1},\ldots, e_{i,m})$ 
            %	        	is the unit vector for dimension $i$}
            \eIf{ $|C^{(t)}| \geq k- \sum_{i=1}^m y_i^{(t-1)} $}{
                $C_{s} \leftarrow $ the largest $k-\sum_{i=1}^m y_i^{(t-1)}$ elements of $C^{(t)}$\;\label{line:tiebreaking}
                $\vec{y}^g \leftarrow \vec y^{(t-1)}+\sum_{j\in C_s}\vec e_j$ 
                \tcp*{$\vec e_j$ is the unit vector}
                {\bf break}\;
                
            }{
                $\vec y^{(t)} \leftarrow \vec y^{(t-1)} + \sum_{j\in C^{(t)}}\vec e_j$\;\label{line:greedye} %\label{line:update}  %\tcp*{greedy update}
        }}        
    }
    {\bf return} $\vec y^g$\;
\end{algorithm}

We use the greedy approach shown in Algorithm~\ref{alg:greedy} to find the optimal integral solution. 
We start from a base vector $\vec{y}^{(0)}$ (line~\ref{line:initialization}) with ${y}^{(0)}_i\ge 1$ and $\sum_{i=1}^m {y}^{(0)}_i \le k$.
In every step, we only increment $y_i$ by one for the dimension $i$ that 
leads to the most decrease in the objective function  (lines~\ref{line:greedyb}--\ref{line:greedye}).
We find out that setting the base vector $\vec{y}^{(0)}$ as the floors of the real-valued solution is not safe, and
we may miss the optimal integral solution.
We could set the base vector to all ones to be safe, but that may cause $k-m$ greedy steps.
Since the input $k$
only needs $\log k$ bits, this leads to a running time exponential to the input size.
By a careful analysis, we find that ${y}^{(0)}_i=\max(1, \left\lfloor Z n_i \sqrt{\theta_i} \cdot k - \sum_{j\ne i} \delta_j \right\rfloor )$ is a tight and safe base (computed in lines~\ref{line:deltab}--\ref{line:deltae}) where
$\delta_i < 1$ is the slack that the $i$-th dimension can contribute in pushing down the base 
of other dimensions, and the downward move $\sum_{j\ne i} \delta_j$ could be $\Theta(m)$ (Lemma~\ref{lem:tight}). 
This results in at most $O(\min(m^2, k))$ greedy steps, and each step takes $O(\log m)$ time when using
a priority queue.
Thus, we have a polynomial time algorithm with running time $O(\min(m^2,k) \log m)$.
Finally, if there are multiple choices in the final greedy step, we choose the one containing the largest
indices (line~\ref{line:tiebreaking}).
Before showing that \alg{GreedyOSA} outputs the leading optimal solution (Theorem~\ref{thm:OSA}), we first introduce a simple lemma as follows.

\begin{lemma} \label{lem:floorminus1}
    There does not exist
    an optimal solution $\tilde{\vec{y}}$ such that there exist $i,j\in [m]$, 
    $\tilde{y}_i \le \left\lfloor Z  n_i \sqrt{\theta_i} \cdot k \right\rfloor - 1$, and
    either (a) $\tilde{y}_j \ge \left\lceil Z  n_j \sqrt{\theta_j} \cdot k \right\rceil + 1$, or
    (b)  $\tilde{y}_j = \left\lceil Z  n_j \sqrt{\theta_j} \cdot k \right\rceil$ and
    $\left\lceil Z  n_j \sqrt{\theta_j} \cdot k \right\rceil 
    \left (\left\lceil Z  n_j \sqrt{\theta_j} \cdot k \right\rceil - 1\right)\ge
    (Z  n_j \sqrt{\theta_j} \cdot k)^2$.
\end{lemma}
\begin{proof}
    Suppose, for a contradiction, that such $\tilde{\vec{y}}$ exists.
    Since $\tilde{\vec y}$ is an optimal solution, $\vec y+ \vec e_i-\vec e_j$ should be no better than 
    $\tilde{\vec y}$, that is
    \begin{align*}
    \frac{n_i^2\theta_i}{ \tilde{y}_i}+\frac{n_j^2\theta_j}{ \tilde{y}_j} \leq \frac{n_i^2\theta_i}{ (\tilde{y}_i+1)}+\frac{n_j^2\theta_j}{ (\tilde{y}_j-1)}.
    \end{align*}
    Thus
    \begin{align}\label{eq3}
    \frac{n_i^2\theta_i}{ \tilde{y}_i(\tilde{y}_i+1)}\leq\frac{n_j^2\theta_j}{ \tilde{y}_j(\tilde{y}_j-1)}.
    \end{align}
    From the condition on  $\tilde{y}_j$ given in the lemma, we know that: either 
    (a) if $\tilde{y}_j \ge \left\lceil Z  n_j \sqrt{\theta_j} \cdot k \right\rceil + 1$, then
    $(Z  n_j \sqrt{\theta_j} \cdot k)^2 \le \left(\lceil Z  n_j \sqrt{\theta_j} \cdot k \rceil +1 \right)
    \lceil Z  n_j \sqrt{\theta_j} \cdot k \rceil \le \tilde{y}_j (\tilde{y}_j-1)$; or
    (b) $(Z  n_j \sqrt{\theta_j} \cdot k)^2 \le \left\lceil Z  n_j \sqrt{\theta_j} \cdot k \right\rceil 
    \left (\left\lceil Z  n_j \sqrt{\theta_j} \cdot k \right\rceil - 1\right) = \tilde{y}_j (\tilde{y}_j-1)$.
    That is, we always have  $(Z  n_j \sqrt{\theta_j} \cdot k)^2 \le \tilde{y}_j (\tilde{y}_j-1)$.
    Then
    \begin{align*}
    \frac{n_i^2\theta_i}{ \tilde{y}_i(\tilde{y}_i+1)}
    \geq   \frac{n_i^2\theta_i}{ \left( \lfloor Z  n_i \sqrt{\theta_i} \cdot k \rfloor -1\right)
        \lfloor Z  n_i \sqrt{\theta_i} \cdot k \rfloor}
    >\frac{n_i^2\theta_i}{ \left( Z  n_i \sqrt{\theta_i} \cdot k\right)^2}
    =\frac{n_j^2\theta_j}{ \left( Z  n_j \sqrt{\theta_j} \cdot k \right)^2}
    \ge \frac{n_j^2\theta_j}{ \tilde{y}_j (\tilde{y}_j-1)},
    \end{align*}
    which contradicts to Inequality~(\ref{eq3}).
    Therefore, the lemma holds.
\end{proof}

\begin{restatable}{theorem}{OSA} \label{thm:OSA}
    Algorithm~\ref{alg:greedy} solves the OSA problem for any input $\vec{\theta}\in [0,1]^m$, 
    i.e., %its output
    $\vec{y}^g \in \argmin_{\vec{y}\in \mathbb{Z}_+^m \land \norm{\vec{y}}_1 \le k} \sum_{i=1}^{m} \frac{n_i^2 \theta_i}{ y_i}$. 
    Moreover, $\vec{y}^g$ is the leading optimal solution.
    The running time of the algorithm is $O(\min(m^2, k) \log m)$.
    %	Moreover, $\vec{y}^o$ is the first in the lexicographical order among all optimal solutions.
    %    \begin{align*}
    %     \min_{\vec y \in (\mathbb{Z^+})^m}   & h(\vec y):=\sum_{i=1}^{m}n_i^2\theta_i^2/y_i \\
    %    \mathrm{s.t.\ }  & \sum_{i=1}^m y_i \leq k.
    %    \end{align*}
\end{restatable}
%\begin{proof}[Proof Outline]
%	The correctness of the algorithm is derived from the following two claims: 
%	(a) for any $\tilde{\vec{y}} \in \argmin_{\vec{y}\in \mathbb{Z}_+^m, \norm{\vec{y}}_1 \le k} \sum_{i=1}^{m} \frac{n_i^2 \theta_i}{ y_i}$, 
%	for any $i\in [m]$, $\tilde{y}_i \ge 
%	\left\lfloor Z n_i \sqrt{\theta_i}\cdot k -\sum_{j\ne i} \delta_j \right\rfloor$, and
%	(b) if the greedy algorithm starts from any $\vec y^{(0)} \in \mathbb{Z}_+^m$ with
%	$\sum_{i=1}^m y^{(0)}_i \leq k$, then it will output the leading optimal
%	solution for the following optimization problem:
%	\begin{align*}
%	\min_{\vec y \in \mathbb{Z}_+^m}  \sum_{i=1}^{m} \frac{n_i^2\theta_i}{y_i }, 
%	\mbox{\ subject\ to\ (a) }  \sum_{i=1}^m y_i \leq k, \mbox{ and (b)\ } 
%	 y_i \geq y_i^{(0)}, \forall i\in[m]. 
%	\end{align*}
%\end{proof}
\begin{proof}
    The correctness of the algorithm is derived from the following two claims: 
    (a) for any $\tilde{\vec{y}} \in \argmin_{\vec{y}\in \mathbb{Z}_+^m, \norm{\vec{y}}_1 \le k} \sum_{i=1}^{m} \frac{n_i^2 \theta_i}{ y_i}$, 
    for any $i\in [m]$, $\tilde{y}_i \ge 
    \left\lfloor Z n_i \sqrt{\theta_i}\cdot k -\sum_{j\ne i} \delta_j \right\rfloor$, and
    (b) if the greedy algorithm starts from any $\vec y^{(0)} \in \mathbb{Z}_+^m$ with
    $\sum_{i=1}^m y^{(0)}_i \leq k$, then it will output the leading optimal
    solution for the following optimization problem:
    \begin{align}
    \min_{\vec y \in \mathbb{Z}_+^m}\quad   & \sum_{i=1}^{m} \frac{n_i^2\theta_i}{y_i }, \label{proof: lp}\\
    \mathrm{subject\ to}\quad  & \sum_{i=1}^m y_i \leq k,
    y_i \geq y_i^{(0)}, \forall i\in[m]. \nonumber
    \end{align}
    
    We now prove Claim (a). 
    Suppose, for a contradiction, that there exists an
    optimal solution $\tilde{\vec y}= (\tilde{y}_1, \tilde{y}_2, \dots, \tilde{y}_m)$, 
    and some $i\in [m]$ such that $\tilde{y}_i \leq \left
    \lfloor Z n_i \sqrt{\theta_i}\cdot k -\sum_{j\ne i} \delta_j  \right\rfloor - 1$.
    Then there must exist some $j\neq i$ such that $\tilde{y}_j \geq 
    \left\lceil Z  n_j \sqrt{\theta_j} \cdot k \right\rceil$.
    Otherwise, $\sum_{j=1}^m \tilde{y}_j < \sum_{j=1}^m Z  n_j \sqrt{\theta_j} \cdot k -1 = k - 1$.
    This means the budget is not fully utilized.
    Then
    $\forall l\in[m]$, $\tilde{\vec y} +\vec e_l$ will be a strictly better solution than $\tilde{\vec y}$, 
    which contradicts to that $\tilde{\vec y}$ is an optimal solution.
    
    Consider every $j$ with  $\tilde{y}_j \geq 
    \left\lceil Z  n_j \sqrt{\theta_j} \cdot k \right\rceil$.
    Since  $\tilde{y}_i \leq \left
    \lfloor Z n_i \sqrt{\theta_i}\cdot k -\sum_{j\ne i} \delta_j  \right\rfloor - 1
    \le \lfloor Z n_i \sqrt{\theta_i}\cdot k \rfloor - 1$, 
    by Lemma~\ref{lem:floorminus1}, it must be that
    $\tilde{y}_j = \left\lceil Z  n_j \sqrt{\theta_j} \cdot k \right\rceil$ and
    $\left\lceil Z  n_j \sqrt{\theta_j} \cdot k \right\rceil 
    \left (\left\lceil Z  n_j \sqrt{\theta_j} \cdot k \right\rceil - 1\right) <
    (Z  n_j \sqrt{\theta_j} \cdot k)^2$.
    By the definition of $\delta_j$, we know that  $\delta_j = \tilde{y}_j - Z  n_j \sqrt{\theta_j} \cdot k $.
    Then we have
    \begin{align*}
    \sum_{j=1}^m \tilde{y}_j 
    & =   \tilde{y}_i + 
    \sum_{j: \tilde{y}_j \geq 
        \left\lceil Z  n_j \sqrt{\theta_j} \cdot k \right\rceil} \tilde{y}_j +
    \sum_{j\ne i: \tilde{y}_j < 
        \left\lceil Z  n_j \sqrt{\theta_j} \cdot k \right\rceil} \tilde{y}_j \\
    & < \left\lfloor Z n_i \sqrt{\theta_i}\cdot k -\sum_{j\ne i} \delta_j  \right\rfloor - 1 + 
    \sum_{j: \tilde{y}_j \geq 
        \left\lceil Z  n_j \sqrt{\theta_j} \cdot k \right\rceil} 
    (Z  n_j \sqrt{\theta_j} \cdot k + \delta_j)  
    + \sum_{j\ne i: \tilde{y}_j < 
        \left\lceil Z  n_j \sqrt{\theta_j} \cdot k \right\rceil} Z  n_j \sqrt{\theta_j} \cdot k\\
    & \le \sum_{j=1}^m Z  n_j \sqrt{\theta_j} \cdot k -1- \sum_{j\ne i} \delta_j + 
    \sum_{j: \tilde{y}_j \geq 
        \left\lceil Z  n_j \sqrt{\theta_j} \cdot k \right\rceil} \delta_j \\
    & \le k-1-\sum_{j\ne i: \tilde{y}_j < 
        \left\lceil Z  n_j \sqrt{\theta_j} \cdot k \right\rceil} \delta_j \\
    & \le k-1.
    \end{align*}
    This again means that the budget is not fully utilized by $\tilde{\vec y}$, and thus $\tilde{\vec y}$
    cannot be an optimal solution, a contradiction.
    
    We now prove Claim (b). 
    We define 
    \begin{align*}
    M_{i,j}=h(\vec{\theta}; \vec y^{(0)}+(j-1)\vec e_i)-h(\vec{\theta}; \vec y^{(0)}+j\vec e_i)=\frac{n_i^2\theta_i}{ \left(y_i^{(0)}+j-1\right)}-\frac{n_i^2\theta_i}{\left(y_i^{(0)}+j\right)}.
    \end{align*}
    Then
    $M_{i,1}>M_{i,2}>M_{i,3}>\cdots$ for any $i\in [m]$.
    
    For any $\vec y$ with $y_i>y_i^{(0)}$ for all $i\in [m]$,
    \begin{align*}
    h(\vec{\theta}; \vec y)=h(\vec{\theta}; \vec y^{(0)})-\sum_{i=1}^m \sum_{j=1}^{y_i-y_i^{(0)}}M_{i,j}.
    \end{align*}
    Thus Problem~(\ref{proof: lp}) can be written as the following problem:
    \begin{align*}
    \max\quad   & \sum_{i=1}^m \sum_{j=1}^{z_i}M_{i,j}\\
    \mathrm{subject\ to}\quad  & \sum_{i=1}^m z_i \leq k-\sum_{i=1}^m y_i^{(0)}\\
    & z_i \geq 0, \forall i\in[m]. 
    \end{align*}
    The solution $\vec{y}$ to Problem~(\ref{proof: lp}) satisfies $y_i = z_i + y_i^{(0)}$ for each $i\in [m]$,
    where $\vec{z}=(z_1, \ldots, z_m)$ is a solution to the above problem.
    
    Since $M_{i,1}>M_{i,2}>M_{i,3}>\cdots$ for every $i\in [m]$, the 
    above problem is equivalent to find the $k-\sum_{i=1}^m y_i^{(0)}$ maximum elements in $\{M_{i,j}\}$. 
    We can first find the maximum elements in $(M_{1,1}, M_{2,1},\ldots, M_{m,1})$, i.e.,
    $C^{(1)} = \argmax_{i\in [m]} M_{i,1}$, and then replace all elements $M_{i,1}$ where $i\in C^{(1)}$ with
    $M_{i,2}$ and find the new maximum elements $C^{(2)}$ in the second iteration, and continue this process
    until we find enough elements. 
    This is exactly the \alg{GreedyOSA} algorithm given in Algorithm~\ref{alg:greedy}.
    
    When there are more than one optimal solutions, in the tie-breaking 
    step (line~\ref{line:tiebreaking}), any subset $C_s'$ of size $k-\sum_{i=1}^m y_i^{(t-1)}$ leads
    to one optimal solution, and every optimal solution is from such a subset.
    Since in line~\ref{line:tiebreaking} we take the elements with the largest values, that means
    the largest $k-\sum_{i=1}^m y_i^{(t-1)}$ dimensions have their $\vec{y}^{(t-1)}$ values incremented
    by $1$, while any other such subset will cause the increment of some other dimension with a smaller index.
    Therefore, the output by \alg{GreedyOSA} is the one lexicographically ordered the first among
    all optimal solutions.
    
    Finally, as for the running time, we know that $\delta_j < 1$, and thus $\sum_{j\ne i} \delta_j = O(m)$.
    Thus, the \alg{GreedyOSA} algorithm starts from a budget of $k - O(m^2)$ and ends when the budget
    $k$ is used up.
    Therefore, the greedy algorithm needs
    at most $O(\min\{m^2, k\})$ steps.
    When we use a priority queue for selecting the maximum value in line~\ref{line:max}, each greedy step
    takes $O(\log m)$ time.
    Therefore, the running time of \alg{GreedyOSA} is $O(\min\{m^2, k\} \log m)$.
\end{proof}

\begin{lemma}\label{lem:tight}
    %The starting base $y^{(0)}_i = \max(1, \left\lfloor \frac{n_i \sqrt{\theta_i}}{\sum_{j=1}^m n_j \sqrt{\theta_j}} \cdot k \right\rfloor -m+2)$ is tight, 
    There exists some problem instance
    in which the optimal solution has $y^{(0)}_i = 
    \left\lfloor Z n_i \sqrt{\theta_i} \cdot k - \sum_{j\ne i} \delta_j \right\rfloor$ for some $i\in [m]$,
    and $\sum_{j\ne i} \delta_j = 0.5(m-1) = \Theta(m)$.
\end{lemma}
\begin{proof}
    For convenience, let $a_i = n_i \sqrt{\theta_i}$.
    For some positive integer $c$, we set
    $a_1 = c(m-1)$, and $a_2 = \cdots =a_m = 1$.
    This can be achieved by properly setting $\{n_i\}$'s and $\{\theta_i\}$'s.
    Set the sample budget $k = 1.5(c+1)(m-1)$.
    We could set $m$ as an odd number so $k$ is an integer.
    For this instance, the real-valued optimal solution based on the Cauchy-Schwartz Inequality is
    $\alpha_i = Z n_i \sqrt{\theta_i} \cdot k = 
    \frac{a_i}{(c+1)(m-1)}\cdot k$.
    For $i\ge 2$, we have $\alpha_i = 1.5$.
    Note that $\lceil \alpha_i \rceil (\lceil \alpha_i \rceil - 1) = 2 < \alpha_i^2$, so
    $\delta_i = \lceil \alpha_i \rceil - \alpha_i = 0.5$ for all $i\ge 2$.
    For $i=1$, $\alpha_1 = \frac{c}{c+1} \cdot (1.5(c+1)(m-1) ) = 1.5c(m-1)$, and thus
    $\lfloor \alpha_1 \rfloor = \alpha_1 = 1.5c(m-1)$.
    We now claim that $\tilde{\vec y} = (\lfloor \alpha_1 - \sum_{j\ne 1} \delta_j \rfloor, 2, 2, \ldots, 2)
    = (\alpha_1 - 0.5(m-1), 2, 2, \ldots, 2)$ is the unique integral
    optimal solution for the above problem instance.
    This would prove the lemma.
    
    First, we have $\sum_{i=1}^m \tilde{y}_i = \alpha_1 -  0.5(m-1) + 2(m-1)
    = 1.5 c(m-1) + 1.5 (m-1) = 1.5(c+1)(m-1) = k$.
    Thus $\tilde{\vec y}$ is a feasible solution and it fully utilizes the budget $k$.
    
    We next prove that for any $i,j \in [m]$, $i\ne j$, $\tilde{\vec y}$ is strictly better than
    $\vec{y}' = \tilde{\vec y} - \vec{e}_i + \vec{e}_j$, when $c$ is large enough.
    In the first case, we consider $i\ne 1$ and $j \ne 1$. 
    Then 
    \begin{align*}
    h(\vec{\theta}; \tilde{\vec y} ) - h(\vec{\theta}; \tilde{\vec y} - \vec{e}_i + \vec{e}_j) 
     = \frac{a_i^2}{\tilde{y}_i} + \frac{a_j^2}{\tilde{y}_j} - 
    \left(  \frac{a_i^2}{\tilde{y}_i - 1} + \frac{a_j^2}{\tilde{y}_j + 1} \right) 
     = \frac{1}{2} + \frac{1}{2} - (1 + \frac{1}{3} )  < 0.
    \end{align*}
    Thus, $\tilde{\vec y}$ is a strictly better solution than $\tilde{\vec y} - \vec{e}_i + \vec{e}_j$.
    In the second case, we have $i=1$ and $j \ne 1$.
    Then $y'_1 = \lfloor \alpha_1 \rfloor - 0.5(m-1) - 1$, and $y'_j = 3 = \lceil \alpha_j \rceil + 1$.
    By Lemma~\ref{lem:floorminus1}, we know that $\vec{y}'$ cannot be the optimal solution, so
    $h(\tilde{\vec y} ) < h(\tilde{\vec y} - \vec{e}_i + \vec{e}_j)$.
    In the third and final case, we have $i \ne 1$ and $j = 1$.
    Then we have
    \begin{align*}
     h(\vec{\theta}; \tilde{\vec y} ) - h(\vec{\theta}; \tilde{\vec y} - \vec{e}_i + \vec{e}_j) 
    & = \frac{a_i^2}{\tilde{y}_i} + \frac{a_j^2}{\tilde{y}_j} - 
    \left(  \frac{a_i^2}{\tilde{y}_i - 1} + \frac{a_j^2}{\tilde{y}_j + 1} \right) \\
    & = \frac{a_1^2}{\tilde{y}_1(\tilde{y}_1+1)} - \frac{a_i^2}{\tilde{y}_i(\tilde{y}_i-1)} \\
    & = \frac{c^2(m-1)^2}{(1.5c(m-1) - 0.5(m-1))(1.5c(m-1) - 0.5(m-1) + 1)} - \frac{1}{2}  \\
    & = \frac{8c^2(m-1)^2 - \left( 9c^2(m-1)^2 -6c(m-1)^2 + (m-1)^2 + (6c-2)(m-1)\right) }{2(3c(m-1) - (m-1))(3c(m-1) - (m-1) + 2)} \\
    & = \frac{- c^2(m-1)^2  + 6c(m-1)^2 - (m-1)^2 - (6c-2)(m-1)}{2(3c(m-1) - (m-1))(3c(m-1) - (m-1) + 2)}.
    \end{align*}
    It is clear that for a large enough $c$, the above formula is negative, which means
    $h(\vec{\theta}; \tilde{\vec y} ) < h(\vec{\theta}; \tilde{\vec y} - \vec{e}_i + \vec{e}_j)$.
    
    Therefore, we know that, for a large enough $c$, making any local change on $\tilde{\vec y}$ by
    decrementing one dimension by $1$ and incrementing another dimension by $1$ will always lead to strictly worse
    solutions.
    Then we claim that $\tilde{\vec y}$ must be the unique optimal solution.
    Suppose, there is another optimal solution $\vec{y}'$.
    We can move from $\tilde{\vec y}$ to $\vec{y}'$ by a series of local change steps.
    In each step, only two dimensions change, and one is incremented by $1$ and the other is decremented by $1$.
    Furthermore, on every dimension $i$, if $\tilde{y}_i > y'_i$, then only decrements occur on this dimension, and
    if $\tilde{y}_i < y'_i$, only increments occur in this dimension.
    Since $\vec{y}'$ is another optimal solution, we have
    $h(\vec{\theta}; \vec{y}') \le h(\vec{\theta}; \tilde{\vec y})$.
    Thus, in at least one local step in the above series, the objective function does not increase.
    Suppose this local step is from $\vec y$ to $\vec y - \vec{e}_i + \vec{e}_j$.
    Thus, we have $h(\vec{\theta}; \vec y) \ge h(\vec{\theta}; \vec y - \vec{e}_i + \vec{e}_j)$.
    By the above argument, we also have $\tilde{y}_i \ge y_i > y_i - 1 \ge y'_i$, and
    $\tilde{y}_j \le y_j < y_j + 1 \le y'_j$.
    Then we have
    \begin{align*}
    0  \le h(\vec{\theta}; \vec y) - h(\vec{\theta}; \vec y - \vec{e}_i + \vec{e}_j) 
     = \frac{a_j^2}{{y}_j({y}_j+1)} - \frac{a_i^2}{{y}_i({y}_i-1)} 
     \le \frac{a_j^2}{\tilde{y}_j(\tilde{y}_j+1)} - \frac{a_i^2}{\tilde{y}_i(\tilde{y}_i-1)} 
     = h(\vec{\theta}; \tilde{\vec y}) - h(\vec{\theta}; \tilde{\vec y} - \vec{e}_i + \vec{e}_j).
    \end{align*}
    However, the above contradicts the claim we proved above that any local change on $\tilde{\vec y}$
    strictly increases the objective function.
    Therefore, we cannot find any other optimal solution, and $\tilde{\vec y}$ is the unique optimal
    solution.
\end{proof}

\begin{restatable}{lemma}{bimonotonicityofosa}\label{lem:bimonotonicity}
    The leading optimal solution $\g(\vec\theta)$ of OSA satisfies the \monotonicity.
\end{restatable}
\begin{proof}
    Recall the proof of Theorem~\ref{thm:OSA}. In claim (b), if we let the initial point $\vec y^{(0)}$ to be $(1,1,\dots,1)$, then the OSA problem is equivalent to find the $k-m$ maximum elements in $\{M_{i,l}\}$, where $M_{i,l}=\frac{n_i^2\theta_i}{ l}-\frac{n_i^2\theta_i}{l+1}$.
    For any $i\in [m]$, if $\theta_i$ increases,  $M_{i,l}$ for any $l$ will also increase.
    This means the $k-m$ maximum elements will contain at least as many elements of $\{M_{i,\cdot}\}$ as before the increase.
    Thus $\g_i(\vec\theta)$ will not decrease.
    Similarly, 
    if $\theta_j$ decreases where $j\neq i$, $M_{j,l}$ for any $l$ will also decrease.
    This means the $k-m$ maximum elements will contain the same number of less elements of $\{M_{j,\cdot}\}$ and thus the same number or more elements of $\{M_{i,\cdot}\}$.
    Therefore, $\g_i(\vec\theta)$ will not decrease.
\end{proof}

\section{Proofs for Section~\ref{sec:appcpel}: Applying COCI to CPE-L}\label{app:applycpe-l}

To align CPE-L to CPE-CS, we treat $\cY\subseteq \{0,1\}^m$ equivalently as a collection of subsets of $[m]$, and
use $\vec{y} \in \cY$ and $S_{\vec{y}} \in \cY$ interchangeably.
For any vector $\vec{\theta}$ and any subset $S\subseteq [m]$, we define
$\vec{\theta}(S) = \sum_{i\in S} \theta_i$.
To prove Lemma~\ref{lem:LambdaDelta}, we first introduce the terminologies used in \cite{chen2014combinatorial}. 

An exchange set $b$ is an ordered pair of disjoint sets $b=(b_+,b_-)$ where $b_+\cap b_- = \emptyset$ and $b_+,b_-\subseteq [m]$.
Then, we define operator $\oplus$ such that, for any set $S \subseteq [m]$ and any exchange set $b=(b_+,b_-)$, we have $S\oplus b \triangleq (S \setminus b_-) \cup b_+$.
Similarly, we also define operator $\ominus$ such that $S \ominus b \triangleq (S \setminus b_+) \cup b_-$.
Define $\width(b) = |b_-| + |b_+|$.

A collection of exchange sets $\cB$ is an \emph{exchange class for $\cY$} if $\cB$ satisfies the following property.
For any $S,S'\in \cY$ such that $S\not = S'$  and for any $i \in (S\setminus S')$, there exists an exchange set $b =(b_+,b_-)\in \cB$ which satisfies five constraints: \textbf{(a)} $i \in b_-$, \textbf{(b)} $b_+ \subseteq S'\setminus S$, \textbf{(c)} $b_- \subseteq S \setminus S'$, \textbf{(d)} $(S\oplus b) \in \cY$ and \textbf{(e)} $(S'\ominus b) \in \cY$.
The width of the exchange class $\cB$ is defined as $\width(\cB) = \max_{b  \in \cB} \width(b)$.
%
%\begin{equation}
%\label{eq:width}
%\width(\cB) = \max_{b  \in \cB} \width(b).
%\end{equation}
%
Let $\Exchange(\cY)$ denote the family of all possible exchange classes for $\cY$.
Then the width of $\cY$ is defined as $\width(\cY) = \min_{\cB \in \Exchange(\cY)} \width(\cB)$.

{\LambdaDelta*}
\begin{proof}
    For any $\vec{\theta}$ with $\g_i(\vec{\theta}) \ne \g_i(\Realtheta)$, we claim that
    $\norm{\vec{\theta} - \Realtheta}_\infty \ge \Delta_i/\width(\cY)$.
    If this claim holds, by the definition of $\Lambda_i$ (Definition~\ref{def:Lambdai}) and $\hardL$, 
    we have $\Lambda_i = \inf_{\vec\theta: \g_i(\vec\theta) \neq \g_i(\Realtheta)} \left\| \vec\theta-\Realtheta \right\|_\infty
    \ge \Delta_i/\width(\cY)$, and thus $\hardL \le \hardD \cdot \width(\cY)^2$ holds.
    
    We now prove the claim by contradiction, that is, we assume that 
    there exists some $\vec\theta$ with $\g_i(\vec{\theta}) \ne \g_i(\Realtheta)$ such that
    $\norm{\vec{\theta} - \Realtheta}_\infty < \Delta_i/\width(\cY)$.
    Let $\cB \in \argmin_{\cB' \in \Exchange(\cY)} \width(\cB')$.
    We discuss two cases separately: $i \in S_{\vec{y}^*}$ and $i \notin S_{\vec{y}^*}$.
    
    {\bf (Case 1)} If $i \in S_{\vec{y}^*}$, then $\g_i(\Realtheta)=1$, and thus $\g_i(\vec{\theta}) = 0$.
    Let $\vec{y} = \g(\vec{\theta})$.
    Then we have $i\not\in S_{\vec{y}}$.
    By the definition of the exchange class, we know that for $S_{\vec{y}^*}$ and $S_{\vec{y}}$, there exists
    an exchange set $b=(b_+,b_-)\in \cB$ such that \textbf{(a)} $i \in b_-$, 
    \textbf{(b)} $b_+ \subseteq S_{\vec{y}}\setminus S_{\vec{y}^*}$, 
    \textbf{(c)} $b_- \subseteq S_{\vec{y}^*} \setminus S_{\vec{y}}$, 
    \textbf{(d)} $(S_{\vec{y}^*}\oplus b) \in \cY$, and
    \textbf{(e)} $(S_{\vec{y}}\ominus b) \in \cY$.
    
    Let $\vec{y}'$ be the binary vector corresponding to $S_{\vec{y}^*}\oplus b$, i.e., $S_{\vec{y}^*}\oplus b = S_{\vec{y}'}$.
    Then $r(\Realtheta; \vec{y}') = r(\Realtheta; \vec{y}^*) - \Realtheta(b_-) + \Realtheta(b_+)$.
    Since $S_{\vec{y}^*}$ is the unique optimal solution under $\Realtheta$ while $i\in S_{\vec{y}^*}$, 
    $i \not\in  S_{\vec{y}'}$, we have $\Delta_i = r(\Realtheta; \vec{y}^*)- 
    \max_{\vec{y}\in {\cY}, i\not\in S_{\vec y} } r(\Realtheta; \vec y) 
    \le r(\Realtheta; \vec{y}^*) - r(\Realtheta; \vec{y}') = \Realtheta(b_-) - \Realtheta(b_+)$.
    We remark that this is also the interpolation lemma (Lemma 2) in \cite{chen2014combinatorial}, supplementary material.
    On the other hand, we have 
    \begin{align}
    \Realtheta(b_-) - \Realtheta(b_+) & = \sum_{i\in b_-} \realtheta_i - \sum_{i\in b_+} \realtheta_i \nonumber \\
    & = \sum_{i\in b_-} (\realtheta_i - \theta_i) + \sum_{i\in b_-} \theta_i - 
    \sum_{i\in b_+} (\realtheta_i - \theta_i) - \sum_{i\in b_+} \theta_i \nonumber \\
    & \le \sum_{i \in b_- \cup b_+} |\realtheta_i - \theta_i| + \sum_{i\in b_-} \theta_i - \sum_{i\in b_+} \theta_i 
    \nonumber \\
    & < \frac{\Delta_i}{\width(\cY) } \cdot \width(b) + \sum_{i\in b_-} \theta_i - \sum_{i\in b_+} \theta_i 
    \label{eq:Deltaassumption} \\
    & \le \Delta_i + \sum_{i\in b_-} \theta_i - \sum_{i\in b_+} \theta_i
    \label{eq:widthdef} \\
    & \le \Delta_i. \label{eq:exchangelast}
    \end{align}
    
    Note that Inequality~\eqref{eq:Deltaassumption} is by the assumption of 
    $\norm{\vec{\theta} - \Realtheta}_\infty < 	\Delta_i/\width(\cY)$, 
    Inequality~\eqref{eq:widthdef} is because $\width(b) \le \width(\cB) = \width(\cY)$, and
    Inequality~\eqref{eq:exchangelast} is  because
    $\sum_{i\in b_-} \theta_i - \sum_{i\in b_+} \theta_i = \vec{\theta}(S_{\vec{y}}\ominus b) - \vec{\theta}(S_{\vec{y}})$,
    and since $S_{\vec{y}}$ is the optimal solution under $\vec{\theta}$, 
    $\vec{\theta}(S_{\vec{y}}\ominus b) \le \vec{\theta}(S_{\vec{y}})$.
    Therefore, we have $\Realtheta(b_-) - \Realtheta(b_+) < \Delta_i$, which contradicts the conclusion of
    $\Delta_i \le \Realtheta(b_-) - \Realtheta(b_+)$ reached on paragraph earlier.

    {\bf (Case 2)} If $i \not\in S_{\vec{y}^*}$, then $ \g_i(\Realtheta)=0$, and thus $\g_i(\vec{\theta}) = 1$.
    Let $\vec{y} = \g(\vec{\theta})$.
    Then we have $i \in S_{\vec{y}}$.
    Again, by the definition of the exchange class, we know that for $S_{\vec{y}}$ and $S_{\vec{y}^*}$, there exists
    an exchange set $b=(b_+,b_-)\in \cB$ such that \textbf{(a)} $i \in b_-$, 
    \textbf{(b)} $b_+ \subseteq S_{\vec{y}^*}\setminus S_{\vec{y}}$, 
    \textbf{(c)} $b_- \subseteq S_{\vec{y}} \setminus S_{\vec{y}^*}$, 
    \textbf{(d)} $(S_{\vec{y}}\oplus b) \in \cY$, and
    \textbf{(e)} $(S_{\vec{y}^*}\ominus b) \in \cY$.
    Let $\vec{y}'$ be the binary vector corresponding to $S_{\vec{y}^*}\ominus b$, i.e., $S_{\vec{y}^*}\ominus b = S_{\vec{y}'}$.
    Then $r(\Realtheta; \vec{y}') = r(\Realtheta; \vec{y}^*) - \Realtheta(b_+) + \Realtheta(b_-)$.
    Since $S_{\vec{y}^*}$ is the unique optimal solution under $\Realtheta$ while $i \not\in S_{\vec{y}^*}$, 
    $i \in  S_{\vec{y}'}$, we have $\Delta_i = r(\Realtheta; \vec{y}^*)- 
    \max_{\vec{y}\in {\cY}, i \in S_{\vec y} } r(\Realtheta; \vec y) 
    \le r(\Realtheta; \vec{y}^*) - r(\Realtheta; \vec{y}') = \Realtheta(b_+) - \Realtheta(b_-)$.
    
    On the other hand, we have 
    \begin{align*}
    \Realtheta(b_+) - \Realtheta(b_-) & = \sum_{i\in b_+} \realtheta_i - \sum_{i\in b_-} \realtheta_i \\
    & = \sum_{i\in b_+} (\realtheta_i - \theta_i) + \sum_{i\in b_+} \theta_i - 
    \sum_{i\in b_-} (\realtheta_i - \theta_i) - \sum_{i\in b_-} \theta_i \\
    & \le \sum_{i \in b_- \cup b_+} |\realtheta_i - \theta_i| + \sum_{i\in b_+} \theta_i - \sum_{i\in b_-} \theta_i \\
    & < \frac{\Delta_i}{\width(\cY) } \cdot \width(b) + \sum_{i\in b_+} \theta_i - \sum_{i\in b_-} \theta_i \\
    & \le \Delta_i + \sum_{i\in b_+} \theta_i - \sum_{i\in b_-} \theta_i \\
    & \le \Delta_i.
    \end{align*}
    
    Note that the last inequality is because 
    $\sum_{i\in b_+} \theta_i - \sum_{i\in b_-} \theta_i 
    = \vec{\theta}(S_{\vec{y}}\oplus b) - \vec{\theta}(S_{\vec{y}})$, and since
    $S_{\vec{y}}$ is the optimal solution under $\vec{\theta}$, we have $\vec{\theta}(S_{\vec{y}}\oplus b) \le \vec{\theta}(S_{\vec{y}})$.
    Again we reach a contradiction.
\end{proof}